\documentclass{article}

 \usepackage[main, final]{neurips_2026}
\usepackage[utf8]{inputenc} % allow utf-8 input
\usepackage[T1]{fontenc}    % use 8-bit T1 fonts
\usepackage{hyperref}       % hyperlinks
\usepackage{url}            % simple URL typesetting
\usepackage{booktabs}       % professional-quality tables
\usepackage{graphicx} % Required for inserting images
\usepackage{amsfonts}       % blackboard math symbols
\usepackage{nicefrac}       % compact symbols for 1/2, etc.
\usepackage{microtype}      % microtypography
\usepackage{xcolor}         % colors
\usepackage{amsmath,amssymb,amsfonts}
\usepackage{amsthm}
\usepackage{algorithm}
\usepackage{algorithmic}
\usepackage[normalem]{ulem} 
\usepackage{array}      
\usepackage[table]{xcolor}
\usepackage{multirow, multicol}
\usepackage[capitalize,noabbrev]{cleveref}
\usepackage{wrapfig}
\definecolor{methodgreen}{RGB}{235,245,235}
\definecolor{methodblue}{RGB}{232,240,250}
\definecolor{baselinegray}{gray}{0.95}

\theoremstyle{plain}
\newtheorem{theorem}{Theorem}[section]
\newtheorem{proposition}[theorem]{Proposition}
\newtheorem{lemma}[theorem]{Lemma}
\newtheorem{corollary}[theorem]{Corollary}
\theoremstyle{definition}
\newtheorem{definition}[theorem]{Definition}
\newtheorem{assumption}[theorem]{Assumption}
\theoremstyle{remark}
\newtheorem{remark}[theorem]{Remark}

\usepackage[textsize=tiny]{todonotes}
\usepackage{float}

\usepackage{amsmath,amsfonts,bm}

\def\eqref#1{equation~\ref{#1}}
\def\1{\bm{1}}

\DeclareMathAlphabet{\mathsfit}{\encodingdefault}{\sfdefault}{m}{sl}
\SetMathAlphabet{\mathsfit}{bold}{\encodingdefault}{\sfdefault}{bx}{n}

\title{SMI: Ontology-Grounded Semantic IDs for Hierarchy-Aware Medical Concept Representation}

\author{%
\begin{tabular}{c}
\textbf{Ziyang Song}\textsuperscript{1,*},
\textbf{Lia Shen}\textsuperscript{2},
\textbf{Yixuan Li}\textsuperscript{3},
\textbf{Qincheng Lu}\textsuperscript{3}
\\[2pt]
\textbf{Weihao Li}\textsuperscript{4},
\textbf{Haoran Xu}\textsuperscript{5},
\textbf{He Zhu}\textsuperscript{3},
\textbf{Xinyu Wang}\textsuperscript{3},
\textbf{Yue Li}\textsuperscript{3,*}
\\[6pt]
\normalfont\textsuperscript{1}School of Electrical Engineering and Computer Science, Ohio University
\\
\normalfont\textsuperscript{2}Department of Quantitative and Computational Biology, University of Southern California
\\
\normalfont\textsuperscript{3}School of Computer Science, McGill University
\\
\normalfont\textsuperscript{4}School of Engineering, Northwestern University
\\
\normalfont\textsuperscript{5}Department of Electrical \& Computer Engineering, University of Washington
\\[2pt]
\normalfont\textsuperscript{*}Corresponding authors:
\texttt{ziyangs@ohio.edu}, \texttt{yueli@cs.mcgill.ca}
\end{tabular}
}

\begin{document}

\maketitle

\begin{abstract}
Biomedical language models (LMs) provide strong representations for clinical tasks, but they often fail to capture the hierarchical structure among medical concepts. This limitation is especially pronounced for rare and fine-grained clinical codes in long-tailed vocabularies, where hierarchy-agnostic LM representations become less informative as concept frequency decreases. We introduce Semantic Medical Identifier (SMI), a training-free framework for constructing ontology-grounded Semantic IDs and hierarchy-aware medical concept representations. For each diagnosis, procedure, or medication concept, SMI defines its Semantic ID as the coarse-to-fine path through an expert-defined medical ontology. SMI then aggregates frozen biomedical LM embeddings across the corresponding ontology levels to form hierarchy-aware medical concept representations that preserve coarse-to-fine clinical semantics. The base SMI uses uniform aggregation across ontology levels, while a level-weighted variant selects global level weights on validation data. Our theoretical analysis shows that SMI yields more informative medical concept representations than hierarchy-agnostic LM embeddings, particularly for rare and fine-grained clinical concepts. Experiments show that SMI captures meaningful hierarchical structure and improves clinical relation prediction. On MIMIC-IV, SMI improves mortality prediction, readmission prediction, drug recommendation, and rare drug prediction.  On eICU, SMI yields more robust patient representations under cross-hospital distribution shift.
\end{abstract}

\section{Introduction}
The widespread adoption of electronic health record (EHR) systems has produced large-scale clinical data, enabling broad applications of language models (LMs) in healthcare. Biomedical LMs learn generalizable clinical representations for downstream tasks such as diagnosis prediction, drug recommendation, and patient risk stratification \citep{timelygpt, trajgpt, saferxagent}. However, most biomedical LMs remain predominantly data-driven and struggle to capture the hierarchical structure among medical concepts. Clinical coding systems such as the International Classification of Diseases (ICD) and the Anatomical Therapeutic Chemical (ATC) organize diagnoses, procedures, and medications into multi-level taxonomies, ranging from broad categories to fine-grained codes \citep{icd, atc}. Without explicitly modeling these medical ontologies, hierarchy-agnostic LMs encode each medical concept largely from its single-level textual description \citep{gao2018representation}, limiting their ability to preserve coarse-to-fine clinical semantics.

\begin{figure}[t]
    \centering
    \includegraphics[width=\columnwidth]{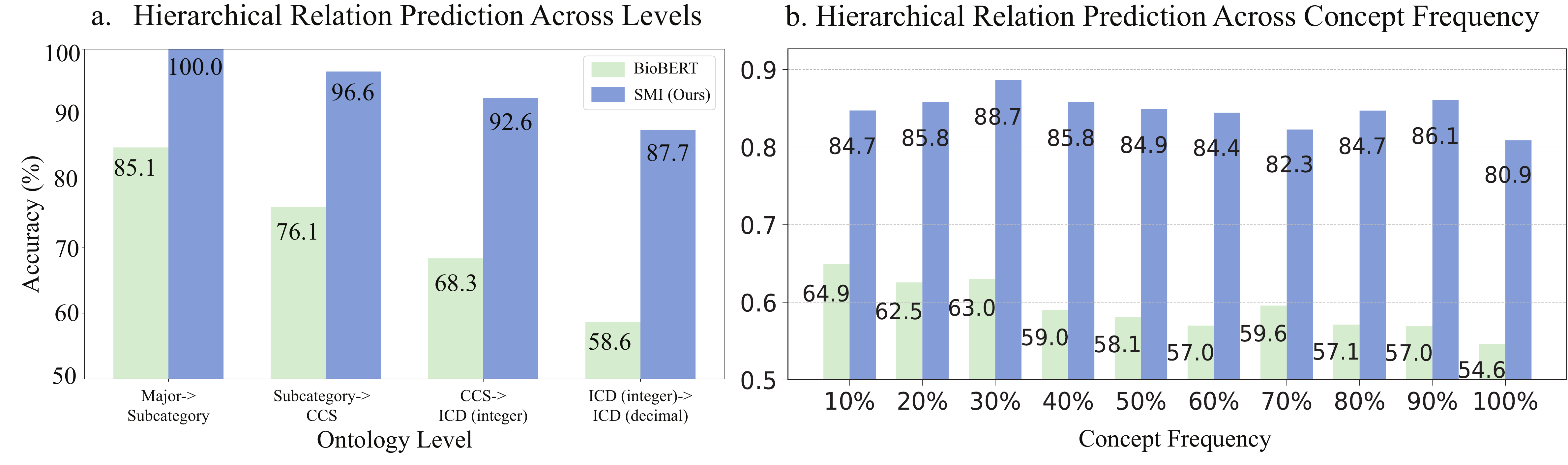}
    \vspace{-1\baselineskip} 
    \caption{
    \textbf{SMI better distinguishes hierarchical relations than BioBERT on ICD 1-hop prediction.} 
    (a) SMI consistently improves hierarchical relation prediction across levels, with the largest gains at finer ICD granularities.
    (b) SMI maintains stronger and more stable performance than BioBERT across concept-frequency bins, suggesting that  medical ontology structure helps mitigate representation degradation for long-tailed clinical codes.    }
    \label{fig:hier_rel_pred}
    \vspace{-1\baselineskip} 
\end{figure}

To demonstrate this limitation, we evaluated ICD 1-hop prediction using BioBERT, a common medical LM pre-trained on PubMed abstracts and PMC articles \citep{lee2020biobert}, and our Semantic Medical Identifier (SMI). This task evaluates whether a child concept is embedded closer to its true parent than to unrelated concepts. As shown in Fig.~\ref{fig:hier_rel_pred}a, BioBERT performs reasonably well at the coarsest levels (85.1\%), but its accuracy deteriorates rapidly as the hierarchy becomes finer, dropping to 58.6\% at the finest level. The drop is most severe for fine-grained concepts, which are far more numerous and rarer in EHR data, making their embeddings harder to learn reliably (Fig.~\ref{fig:long_tail_mimic4}). To further examine this long-tail effect, we stratified fine-grained ICD concepts by frequency and evaluated 1-hop prediction. As shown in Fig.~\ref{fig:hier_rel_pred}b, BioBERT accuracy decreases as concept frequency decreases, indicating that hierarchy-agnostic LM embeddings are vulnerable to long-tailed clinical code distributions. In contrast, SMI maintains strong and stable accuracy across frequency bins, suggesting that medical ontology structure can help mitigate representation degradation for infrequent fine-grained concepts. These observations motivate leveraging expert-defined medical ontologies to improve representations of long-tailed clinical codes under data scarcity.  

Prior structure-aware approaches for medical concept representation broadly fall into two categories. The first learns structure-aware concept embeddings by incorporating curated biomedical resources such as UMLS \citep{liu2021learning}, knowledge graphs \citep{sakhovskiy-et-al-2024-bergamot}, and ontologies \citep{he2024language}, but usually requires additional structure-aware training. The second develops task-specific prediction models that use ontology structure \citep{choi2017gram, shang2019pre, tan2025boxlm}, but these methods are optimized for specific tasks rather than general-purpose concept representation across clinical applications. These limitations leave a gap for a hierarchy-aware framework that directly constructs concept representations without training an explicit ontology encoder or making representation construction dependent on a particular downstream task. SMI addresses this gap by directly aggregating frozen, level-specific LM embeddings along expert-defined medical ontology paths to construct hierarchy-aware medical concept representations without additional training.

% Semantic alignment methods leverage curated biomedical resources (e.g., synonyms, definitions) to learn semantic/ally grounded concept embeddings \citep{remy-etal-2022-biolord,Yuan2022-fj}. However, they mainly capture local semantic similarity and do not explicitly model the global relations among clinical concepts. In contrast, structure‑aware methods leverage structured domain knowledge to learn  representations consistent with clinical relations \citep{he2024language}. However, these methods may underutilize the rich textual semantics captured by LMs, limiting their ability to capture multi-level hierarchical semantics. To bridge both gaps, we propose to construct hierarchy-aware concept embeddings by aggregating LM-encoded representations via expert-curated medical hierarchy, yielding embeddings that are both semantically expressive and hierarchically consistent.

% Semantic Identifier (ID) approach was recently introduced in recommendation systems to represent items as discrete tokens that encode coarse-to-fine semantics \citep{grid, 10.1145/3640457.3688190, . Most approaches learn these IDs in a purely data-driven manner via RQ-VAE with a fixed hierarchy depth and a fixed number of tokens per level \citep{rqvae}. However, because these hierarchies are learned purely from data rather than anchored to expert-defined medical taxonomies, they may misalign with domain knowledge and offer limited interpretability. 

In this study, we introduce SMI, a training-free framework for constructing ontology-grounded Semantic IDs and hierarchy-aware medical concept representations. SMI derives Semantic IDs from expert-defined medical ontologies. Diagnosis and procedure concepts are represented using coarse-to-fine ICD taxonomy, while medication concepts follow the ATC taxonomy. SMI then aggregates frozen biomedical LM embeddings along each Semantic ID to form a hierarchy-aware medical concept representation that preserves coarse-to-fine clinical semantics. The base SMI assigns equal importance to all hierarchy levels, while a level-weighted variant accounts for their different contributions by selecting a global weight for each level on validation data. Theoretically, we show that SMI yields more informative medical concept representations than hierarchy-agnostic LM embeddings, particularly for rare and fine-grained clinical concepts. Empirically, clinical relation prediction experiments show that SMI produces concept representations that better align with medical ontology structure, improving relation prediction within ICD and ATC as well as cross-ontology ICD-to-ATC prediction. On the MIMIC-IV dataset, SMI improves mortality prediction, readmission prediction, and drug recommendation, with particularly strong gains in rare drug prediction. On the eICU dataset, SMI shows significantly smaller cross-hospital distribution shifts than BioBERT, indicating improved cross-site robustness.

\section{Related Work}

Prior work on biomedical concept representation with external knowledge mainly follows two directions. One direction learns structure-aware concept embeddings by injecting curated biomedical resources into representation learning. Existing methods use UMLS synonym sets \citep{sung2020biomedical, liu2021learning}, ontology definitions and descriptions \citep{remy-etal-2022-biolord, Remy2024-hb}, biomedical knowledge graphs \citep{Zou2022-te, Yuan2022-fj, sakhovskiy-et-al-2024-bergamot}, and medical ontologies \citep{he2024language, OnT}. While these methods improve concept embeddings using external biomedical knowledge, they typically require additional training rather than the direct use of frozen LM representations.
The other direction develops supervised downstream prediction models that incorporate ontology structure through ontology-guided aggregation \citep{choi2017gram}, graph-augmented patient modeling \citep{shang2019pre}, expert-guided topic modeling \citep{song2022seed, Yang2026-mb}, or structure-semantics fusion \citep{tan2025boxlm}. These methods are trained for specific predictive tasks, whereas SMI constructs task-agnostic concept representations for downstream clinical applications.
In contrast, SMI uses expert-defined ontology paths to construct hierarchy-aware concept representations from frozen LM embeddings without training.

Semantic IDs represent entities as structured sequences of discrete codes, allowing semantically related entities to share common prefixes and capture coarse-to-fine semantic structure. TIGER constructs hierarchical Semantic IDs by applying a residual-quantized variational autoencoder (RQ-VAE) to item content embeddings \cite{TIGER}. Subsequent work further improves learned Semantic IDs by aligning language and collaborative information \cite{LC-Rec}, incorporating collaborative signals into item tokenization \citep{LETTER}, and jointly learning item tokens with the downstream model \cite{ETEGRec}.
These methods construct discrete identifiers through learned quantization or tokenization. In contrast, SMI derives ontology-grounded Semantic IDs directly from expert-defined medical ontologies, with each Semantic ID corresponding to a medical concept’s coarse-to-fine ontology path. SMI then aggregates frozen biomedical LM embeddings along these ontology paths to construct hierarchy-aware medical concept representations without training an additional quantizer or tokenization model.

\begin{figure}[t]
\centering
\includegraphics[width=\textwidth]{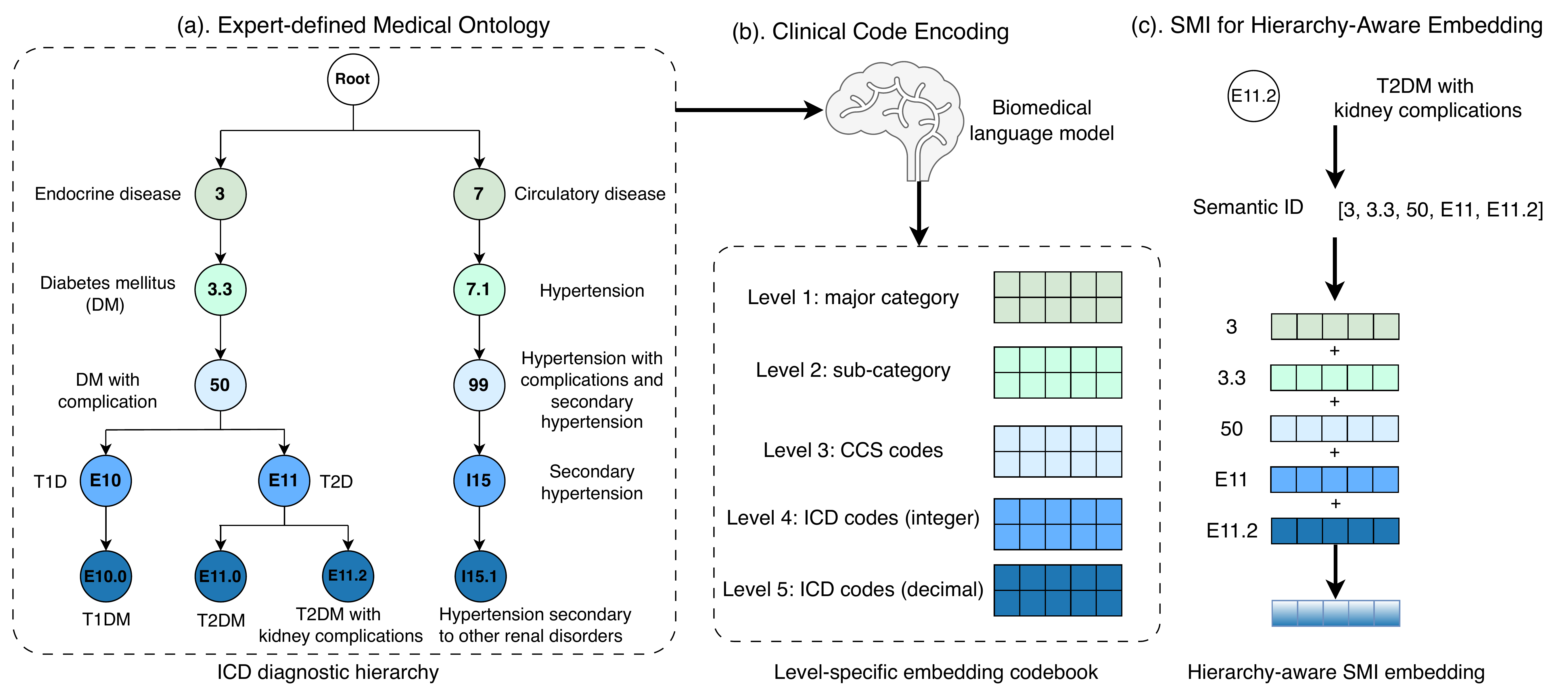} 
\vspace{-1\baselineskip} 
\caption{Overview of SMI and hierarchy-aware medical concept embedding. \textbf{a}. An expert-curated medical ontology organizes clinical concepts from coarse to fine levels. \textbf{b}. A frozen biomedical LM encodes clinical codes at each ontology level into level-specific embeddings. \textbf{c}. SMI represents each medical concept by its coarse-to-fine ontology path as a Semantic ID and aggregates the corresponding embeddings to form a hierarchy-aware medical concept representation.}
\label{fig: SMI}
\vspace{-1\baselineskip} 
\end{figure}

\section{Methodology}

This section presents the SMI framework for hierarchy-aware concept representation of clinical codes, including diagnoses, procedures, and medications. Section~\ref{sec: medical_onto} describes the construction of medical ontologies. Section~\ref{sec: SMI} introduces the SMI framework. Section~\ref{sec: theory} provides a theoretical analysis showing that SMI yields more informative representations than hierarchy-agnostic LMs.

\subsection{Expert-Curated Medical Ontologies}
\label{sec: medical_onto}
% From this path, we generate a sequence of hierarchical indices (e.g., (4, 70, 121)) as \textbf{Semantic ID}, where 4 is the index for "Circulatory Disease" in the first layer and so on. 

% Accurately representing medical concepts requires capturing their inherent hierarchical structure.
Clinical coding systems, such as ICD and ATC, organize diagnoses, procedures, and medications through multi-level taxonomies ranging from broad categories to fine-grained conditions. We organize clinical concepts into distinct five-level hierarchies, where each code has a unique path. For diagnostic codes, we adopt the Clinical Classifications Software (CCS) diagnostic taxonomy and construct a hierarchy of Major category $\rightarrow$ Sub-category $\rightarrow$ CCS code $\rightarrow$ ICD-CM code (integer) $\rightarrow$ ICD-CM code (decimal). For procedure codes, we follow the CCS procedure taxonomy and construct a five-level hierarchy: major category, sub-category, CCS code, prefix ICD-PCS code (first three characters), and full ICD-PCS code. For medications, we leverage the inherent five-level hierarchy of the ATC system: Anatomical Main Group $\rightarrow$ Therapeutic Subgroup $\rightarrow$ Pharmacological Subgroup $\rightarrow$ Chemical Subgroup $\rightarrow$ Chemical Substance. We provide the construction procedure and key statistics of medical ontologies in Appendix~\ref{app: med_onto}.

% For ICD-9, mappings were derived from the official AHRQ release (\textit{Appendix A: Single Level Diagnoses}: \texttt{AppendixASingleDX.txt}), 
% which provides one-to-many associations between ICD codes and major CCS categories. Available at \url{https://hcup-us.ahrq.gov/toolssoftware/ccs/ccs.jsp}).
% \smallskip
% \\For ICD-10, we adopted mappings from the open-source GraphCare repository (\texttt{ICD10CM\_to\_CCSCM.csv} and \texttt{CCSCM.csv}), 
% which align ICD-10 strings and codes to CCS counterparts via both string-based and ID-based matching. 

% \paragraph{Hierarchical CCS Structure.} Beyond the flat major CCS categories, we constructed a multi-level hierarchy reflecting parent-child relationships among disease groups. Using the AHRQ document (\textit{Appendix C: Multi-Level Diagnoses}: \texttt{AppendixCMultiDX.txt}), we built a four-layer taxonomy: \textbf{Major $\rightarrow$ Sub $\rightarrow$ CCS $\rightarrow$ ICD}, where each ICD node has a unique path to its parent CCS node and ultimately to the major category (e.g., \textit{Infectious and Parasitic Diseases $\rightarrow$ Bacterial Infection $\rightarrow$ Tuberculosis $\rightarrow$ 01000}). This structure was stored into a layer-based mapping file (\texttt{icd\_9\_10\_w\_ccs\_combined\_cleaned.csv}) for subsequent embedding alignment.

\subsection{SMI: Semantic IDs for Medical Concept Representation}
\label{sec: SMI}
Clinical concepts are organized in an expert-curated medical ontology with $L$ granularity levels. Let $C^{(\ell)}=\{c^{(\ell)}_{1}, \dots, c^{(\ell)}_{|C^{(\ell)}|}\}$ denote the set of clinical codes at level $\ell \in \{1,\dots,L\}$. As illustrated in Fig.~\ref{fig: SMI}, the \emph{SMI} of concept $i$ is its ontology path in the hierarchy, defined as the ordered sequence of ancestor codes from coarse to fine levels:
\begin{equation}
\mathcal{P}(i) \equiv \big(a_i^{(1)},\dots,a_i^{(L)}\big)
\end{equation}
where $a_i^{(\ell)} \in C^{(\ell)}$ denotes the ancestor of concept $i$ at level $\ell$. 
As shown in Fig.~\ref{fig: SMI}c, the ICD-10-CM code \texttt{E11.2} (type 2 diabetes mellitus with kidney complications) follows the path \texttt{[3, 3.3, 50, E11, E11.2]} corresponding to endocrine disease (Major category), diabetes mellitus with complication (Sub-category), diabetes mellitus with complication (CCS code), type 2 diabetes mellitus (integer-level ICD), and itself (decimal-level ICD).

% \begin{figure*}[!t]
% \centering
% \includegraphics[width=0.7\textwidth]{figures/contrastive learning.png} 
% \caption{LM hyperbolic embedding \textbf{a}. LM embeddings are represented in a $d$-dimensional hyper-cube, without capturing hierarchical structure. \textbf{b}. Using contrastive learning, concepts are organized according to hierarchical semantics within a Poincaré ball of radius $\sqrt{d}$. }
% \label{fig: hyperbolic}
% \vspace{-0.6\baselineskip} 
% \end{figure*}

% \begin{equation}
%     \mathbf{c}_{\text{I10}} = (\underbrace{\texttt{7}}_{\text{Major}}, \underbrace{\texttt{7.1}}_{\text{Sub}}, \underbrace{\texttt{99}}_{\text{CCS}}, \underbrace{\texttt{I10}}_{\text{ICD}}),
% \end{equation}
% representing \textit{Circulatory Disease} $\to$ \textit{Hypertension} $\to$ \textit{Essential Hypertension} $\to$ \textit{Primary Hypertension}.
 
% Specifically, the first token $c^{(1)}$ denotes the coarsest concept of major disease category, while  $c^{(2)}$ and $c^{(3)}$ specify sub-category and CCS code, respectively. The last token $c^{(4)}$ further refines the most fine-grained ICD code. This allows us to control both the amount and the structure of clinical information encoded within each SMI. 

% textual descriptions of clinical codes at four levels of a clinical hierarchy, i.e., major category, sub-category, CCS code, and ICD code (Fig.~\ref{fig: SMI}.b). Let the hierarchy levels be $\ell \in \mathcal{L}=\{1,2,3,4\}$, corresponding to a specific level (e.g., Major category, Sub-category, CCS, ICD) with $K^{(l)}$ clinical codes.

We construct a \textit{hierarchy-aware medical concept representation} for concept $i$ by aggregating level-specific embeddings of the codes in its Semantic ID $\mathcal{P}(i)$. For each clinical code $c^{(\ell)}_j$ at level $\ell$, we use a frozen biomedical LM $\text{Enc}(\cdot)$ to encode its textual description and then obtain its embedding by mean pooling the token embeddings:
\begin{equation}
\mathbf{e}^{(\ell)}_{j}=\text{meanpool}\big(\text{Enc}(c^{(\ell)}_{j})\big)
\end{equation}
Let $\mathbf{E}^{(\ell)} \in \mathbb{R}^{|C^{(\ell)}|\times d}$ denote the level-$\ell$ embedding table, whose row $\mathbf{e}^{(\ell)}_j \in \mathbb{R}^d$ represents code $c^{(\ell)}_j$. We obtain the SMI embedding of concept $i$ by summing the level-specific embeddings along $\mathcal{P}(i)$:
\begin{equation}\label{eq:smi}
\mathbf{e}^{\text{SMI}}_{i}=\sum_{\ell=1}^{L}\mathbf{e}_{a_i^{(\ell)}}^{(\ell)}
\end{equation}
This yields a $d$-dimensional SMI embedding $ \mathbf{e}_i^{\text{SMI}}$ that preserves coarse-to-fine semantics without introducing additional trainable parameters.

The base SMI in \eqref{eq:smi} uses uniform aggregation across hierarchy levels. To account for their different contributions to the final representation, we consider a level-weighted variant, SMI-W, with \emph{global} level weights shared across all concepts. Let ${w_1,\dots,w_L}$ denote the level weights, where $w_\ell \ge 0, \sum_{\ell=1}^{L} w_\ell = L$. We define the weighted SMI representation of concept $i$ as
\begin{align}
    \label{eq:smiw}
    \mathbf{e}_i^{\text{SMI-W}} = \sum_{\ell=1}^{L} w_\ell \, \mathbf{e}^{(\ell)}_{a_i^{(\ell)}}.
\end{align}
SMI-W replaces uniform aggregation with global level weighting. When $w_\ell=1$ for all $\ell$, it reduces to the original SMI. The weights are shared across all concepts and selected on the validation set.

\subsection{Theoretical Analysis of LM Representation Collapse and SMI Informativeness}
\label{sec: theory}

We consider two embedding paradigms for clinical concepts. A \emph{hierarchy-agnostic LM embedding} $\mathbf{e}_i^{\mathrm{LM}}$ is obtained by directly encoding the textual description of concept $i$ with a biomedical LM. In contrast, our proposed \emph{hierarchy-aware SMI embedding} $\mathbf{e}_i^{\mathrm{SMI}}$ aggregates embeddings of all ancestor concepts along its hierarchical path. Both aim to estimate the latent semantic vector $\boldsymbol{\mu}_i$. However, $\mathbf{e}_i^{\mathrm{LM}}$ depends only on data-driven textual signal, whereas $\mathbf{e}_i^{\mathrm{SMI}}$ leverages higher-frequency ancestor information to reduce estimation noise variance and thus improves representations. All proofs in this section are detailed in Appendix~\ref{sec:proof}.

Prior work has shown that word frequency affects contextualized embedding quality in LMs~\citep{schick2020rare,  schick2020bertram}, which is consistent with our observation in Fig.~\ref{fig:hier_rel_pred}b that BioBERT accuracy declines for less frequent concepts. Motivated by these findings, we model hierarchy-agnostic LM embeddings $\mathbf{e}_i^{\mathrm{LM}}$ as frequency-dependent noisy estimates of latent concept semantics $\boldsymbol{\mu}_i$, with estimation variance decreasing as frequency $n_i$ increases. Because clinical vocabularies are long-tailed (Table~\ref{tab:hierarchy_stat}), many fine-grained concepts have small $n_i$ and thus noise-dominated embeddings. This motivates the following frequency-dependent estimation assumption.

\begin{assumption}[Frequency-dependent embedding estimation] \label{ass:embedding_est}
For each concept $i$, let $\boldsymbol{\mu}_i \in \mathbb{R}^d$ denote its true semantic vector and $n_i$ its frequency in the training corpus. We estimate its LM embedding $\mathbf{e}_i^{\mathrm{LM}}$ as a true signal perturbed by a zero-mean isotropic noise:
\begin{align*}
    \mathbf{e}_i^{\mathrm{LM}} = \boldsymbol{\mu}_i + \boldsymbol{\epsilon}_i,
    \qquad 
    \boldsymbol{\epsilon}_i \sim \mathcal{N} \left( \mathbf{0}, \frac{\sigma^2}{n_i}\mathbf{I}_d \right).
\end{align*}
\end{assumption}
Assumption~\ref{ass:embedding_est} implies that LM embedding noise decreases with concept frequency. In long-tailed clinical vocabularies, fine-grained concepts typically occur infrequently and thus suffer from high noise. This motivates a signal-to-noise (SNR) criterion for measuring semantic separability.

\begin{definition}[Pairwise SNR]\label{def:pairwise-snr}
Given two concepts $i$ and $j$ with true semantic vectors $(\boldsymbol{\mu}_i, \boldsymbol{\mu}_j)$ and noises $(\boldsymbol{\epsilon}_i, \boldsymbol{\epsilon}_j)$, we define the pairwise SNR criterion as the ratio of their semantic distance to the expected noise magnitude:
\[
    \mathrm{SNR}(i,j)
    := \frac{\|\boldsymbol{\mu}_i - \boldsymbol{\mu}_j\|}
    {\sqrt{\mathbb{E}\!\left\|\boldsymbol{\epsilon}_i -
    \boldsymbol{\epsilon}_j\right\|^2}}.
\]
Under Lemma~\ref{lem:noise-diff}, this expression simplifies to
$$
    \mathrm{SNR}(i,j)
    = \frac{\|\boldsymbol{\mu}_i - \boldsymbol{\mu}_j\|}
    {\sigma \sqrt{d\!\left(\frac{1}{n_i} + \frac{1}{n_j}\right)}}.
$$
\end{definition}
Under Assumption~\ref{ass:embedding_est} and the long-tailed frequency distribution, fine-grained concepts exhibit vanishing pairwise SNR, leading to the following theorem:

\begin{theorem}[Inseparability of infrequent concepts with hierarchy-agnostic LM embeddings] \label{thm:separability_collapse} 
Let $i$ and $j$ be two fine-grained concepts with low frequencies. The semantic difference between the two concepts becomes negligible relative to the noise:
\[
    \frac{\|\boldsymbol{\mu}_i - \boldsymbol{\mu}_j\|}
    {\|\boldsymbol{\epsilon}_i - \boldsymbol{\epsilon}_j\|}
    \xrightarrow{p} 0,
\]
where $\|\boldsymbol{\epsilon}_i - \boldsymbol{\epsilon}_j\|$ represents the magnitude of noise difference between the two concepts. 
\end{theorem}

Theorem~\ref{thm:separability_collapse} shows that, in long-tailed clinical vocabularies, hierarchy-agnostic LMs produce noise-dominated embeddings for infrequent fine-grained concepts, causing  semantic separability to degrade at finer hierarchy levels.
\begin{remark}
Pairwise distances between infrequent fine-grained concepts concentrate in a narrow ``noise shell,'' making their representations geometrically indistinguishable. Thus, 1-hop prediction with hierarchy-agnostic LM embeddings degrades to near-random results (Fig.~\ref{fig:hier_rel_pred}a).
\end{remark}

Inspired by the additive compositionality of word embeddings \citep{skip_gram}, we model a fine-grained concept as the sum of its coarse-grained parent semantics and a child-specific semantic offset. By leveraging more informative parent representations, SMI mitigates the data scarcity and noise of fine-grained concepts. We formalize this hierarchical semantic decomposition as follows:

\begin{assumption}[Hierarchical semantic decomposition] \label{ass:hier-decom}
For a child concept $c$ and its parent concept $p$, we assume the child’s semantic vector decomposes as
\[
\boldsymbol{\mu}_c^{(\ell+1)} = \boldsymbol{\mu}_p^{(\ell)} +
\boldsymbol{\delta}_c^{(\ell+1)},
\]
where $\boldsymbol{\mu}_p^{(\ell)}$ captures the shared coarse-grained semantics of the parent concept and $\delta_{c}^{(\ell+1)}$ captures the child-specific semantic offset (e.g., severity or complication of a disease). We assume offsets are zero-centered ($\mathbb{E}[\delta] = 0$) and are independent across child concepts.
\end{assumption}

Appendix~\ref{app:offset_consistency} provides empirical offset consistency analysis supporting this decomposition using real ICD-9 parent-child concept pairs. Offsets sharing the same clinical modifier are substantially more aligned than offsets from different modifier groups, with mean cosine similarity $0.655$ versus $0.080$ ($p=2.0\times 10^{-4}$). Under Assumption~\ref{ass:hier-decom}, SMI improves the representation of infrequent fine-grained concepts by anchoring them to higher-frequency ancestor semantics with lower variance, rather than relying only on infrequent fine-grained codes. This explains why SMI can mitigate representation collapse in long-tailed clinical vocabularies. The following proposition formalizes this effect by showing that SMI preserves separability between an infrequent concept and non-parent concepts.

\begin{proposition}[SMI preserves separability against non-parent concepts]
\label{prop:smi_nonparent_separable}
Under Assumptions~\ref{ass:embedding_est} and~\ref{ass:hier-decom}, let $c$ be a child concept with parent $p$, and let $j\neq p$ be any other concept.
If $\|\mu_p-\mu_j\|>0$  then as $n_c\to 0$,
\[
\frac{\|\mu^{\mathrm{SMI}}_c-\mu^{\mathrm{SMI}}_j\|}
{\|\epsilon^{\mathrm{SMI}}_c-\epsilon^{\mathrm{SMI}}_j\|}
\;\xrightarrow{p}\;
\frac{\|\mu_p-\mu_j\|}
{\|\epsilon_p-\epsilon_j\|}
\;>\;0 .
\]
\end{proposition}
Proposition~\ref{prop:smi_nonparent_separable} shows that SMI embeddings are separable from non-parent concepts even for infrequent fine-grained concepts, whereas Theorem~\ref{thm:separability_collapse} indicates that hierarchy-agnostic LM embeddings are inseparable. Although both approaches estimate the same latent semantics, SMI improves representations by reducing embedding variance rather than increasing semantic margins. By aggregating higher-frequency ancestor embeddings, SMI reduces noise in infrequent fine-grained concepts. This leads to the following theorem, which establishes that SMI achieves strictly higher SNR than hierarchy-agnostic LM embeddings by reducing estimation noise variance.

% \begin{theorem}[SMI Improves SNR via Variance Reduction]
% \label{thm:smi_snr_gain}
% Under Assumption~\ref{ass:embedding_est} and a long-tailed frequency distribution, SMI strictly reduces estimation noise:
% \[
% \mathbb{E}\,\|\boldsymbol{\epsilon}_i^{\mathrm{SMI}}\|^2
% <
% \mathbb{E}\,\|\boldsymbol{\epsilon}_i^{\mathrm{LM}}\|^2.
% \]
% Consequently, for any pair $(i,j)$ with fixed semantic margin $\|\boldsymbol{\mu}_i - \boldsymbol{\mu}_j\|$, the pairwise SNR satisfies
% \[
% \mathrm{SNR}_{\mathrm{SMI}}(i,j)
% >
% \mathrm{SNR}_{\mathrm{LM}}(i,j).
% \]
% \end{theorem}

\begin{theorem}[Weighted SMI Improves SNR via Variance Reduction]
\label{thm:smi_snr_gain}
Under Assumption~\ref{ass:embedding_est}, suppose the level-wise estimation noises are independent and at least one ancestor level is more frequent than the leaf concept, i.e., $n_i^{(\ell)} > n_i^{(L)}$ for some $\ell<L$. After normalizing the hierarchy-level weights, weighted SMI strictly reduces  estimation noise variance of infrequent fine-grained concepts:
\[
\mathbb{E}\,\|\boldsymbol{\epsilon}_i^{\mathrm{SMI\text{-}W}}\|^2
<
\mathbb{E}\,\|\boldsymbol{\epsilon}_i^{\mathrm{LM}}\|^2.
\]
Consequently, for any pair $(i,j)$ with fixed semantic margin $\|\boldsymbol{\mu}_i - \boldsymbol{\mu}_j\|$, the pairwise SNR satisfies
\[
\mathrm{SNR}_{\mathrm{SMI\text{-}W}}(i,j)
>
\mathrm{SNR}_{\mathrm{LM}}(i,j).
\]
\end{theorem}

\begin{corollary}[Uniform SMI improves SNR]
\label{cor:uniform_smi_snr}
Uniform SMI is a special case of weighted SMI with equal level weights. Therefore, under the same conditions as Theorem~\ref{thm:smi_snr_gain}, uniform SMI strictly reduces estimation noise variance after normalizing the equal hierarchy-level weights:
\[
\mathbb{E}\,\|\boldsymbol{\epsilon}_i^{\mathrm{SMI}}\|^2
<
\mathbb{E}\,\|\boldsymbol{\epsilon}_i^{\mathrm{LM}}\|^2.
\]
Uniform SMI also achieves higher pairwise SNR than the hierarchy-agnostic LM embedding:
\[
\mathrm{SNR}_{\mathrm{SMI}}(i,j)
>
\mathrm{SNR}_{\mathrm{LM}}(i,j).
\]
\end{corollary}

Uniform SMI is a special case of weighted SMI with equal hierarchy-level weights. Therefore, under the same conditions as Theorem~\ref{thm:smi_snr_gain}, uniform SMI also strictly reduces estimation-noise

\begin{table}[t]
\centering
\footnotesize
\setlength{\tabcolsep}{4.5pt}
\caption{Clinical relation prediction results for ICD$\rightarrow$ICD, ATC$\rightarrow$ATC, and ICD$\rightarrow$ATC, comparing SMI and SMI-W against baselines. Best F-scores are highlighted in bold. We reported the optimized result for each method as the decision threshold is selected via deterministic grid search. }
\label{tab:hier_pred_prf}
\begin{tabular}{lccc ccc ccc}
\toprule
\multirow{2}{*}{\textbf{Method}} &
\multicolumn{3}{c}{\textbf{ICD}} &
\multicolumn{3}{c}{\textbf{ATC}} &
\multicolumn{3}{c}{\textbf{ICD$\rightarrow$ATC}} \\
\cmidrule(lr){2-4}\cmidrule(lr){5-7}\cmidrule(lr){8-10}
& \textbf{Precision} & \textbf{Recall} & \textbf{F-score}
& \textbf{Precision} & \textbf{Recall} & \textbf{F-score}
& \textbf{Precision} & \textbf{Recall} & \textbf{F-score} \\
\midrule
\rowcolor{methodgreen}
\textbf{SMI} & 70.6 & 67.9 & 69.2 & 83.8 & 67.6 & 74.9 & 53.2 & 61.4 & 57.0 \\
\rowcolor{methodgreen}
\textbf{SMI-W} & 77.2 & 66.6 & \textbf{71.5} & 86.1 & 70.5 & \textbf{77.5} & 53.6 & 62.2 & \textbf{57.6} \\
\midrule
BioBERT & 30.4 & 75.8 & 44.2 & 50.7 & 45.7 & 48.1 & 49.3 & 57.0 & 50.9\\
SapBERT & 69.9 & 64.5 & 67.1 & 84.0 & 66.9 & 74.5 & 50.2 & 55.7 &52.9 \\
HiTs & 69.3 & 61.7 & 65.3 & 83.0 & 66.2 & 71.5 & 50.5 & 56.3 & 53.2  \\
OnT & 65.9 & 59.1 & 62.3 & 78.5 &62.3  & 69.5  &48.9 & 55.2 & 51.9 \\
BERGAMOT & 66.5 & 70.2 & 55.3 & 79.4 & 69.7 & 74.2 & 47.9 & 56.2 & 51.7 \\
RotatE & 69.3 & 58.2 & 63.3 & 80.3 & 63.9 & 71.2 & 49.2 & 57.2 & 52.9 \\
cui2vec & 68.3 & 60.3 & 64.1 & 77.2 & 61.4 & 68.4 & 45.7 & 54.0 & 49.5 \\\bottomrule
\end{tabular}
\end{table}

\begin{table}[t]
\centering
\small
\caption{Aggregate clustering quality of ICD diagnosis embeddings by major disease category. NMI, ARI, Purity, Recall, Precision, and Silhouette are higher better, while Intra/Inter Ratio is lower better. SMI achieves stronger cluster-label alignment and better geometric separation than BioBERT, indicating more ontology-consistent diagnosis embeddings.}
\label{tab:cluster_aggregate}
\begin{tabular}{lccccccc}
\toprule
\textbf{Method}
& \textbf{NMI $\uparrow$}
& \textbf{ARI $\uparrow$}
& \textbf{Purity $\uparrow$}
& \textbf{Recall $\uparrow$}
& \textbf{Precision $\uparrow$}
& \textbf{Sil. $\uparrow$}
& \textbf{Intra/Inter $\downarrow$} \\
\midrule
SMI
& \textbf{0.77}
& \textbf{0.56}
& \textbf{0.78}
& \textbf{0.80}
& \textbf{0.61}
& \textbf{0.52}
& \textbf{0.29} \\
BioBERT
& 0.39
& 0.25
& 0.51
& 0.41
& 0.35
& 0.04
& 0.82 \\
\bottomrule
\end{tabular}
\end{table}

\section{Experiments and Results}

\subsection{SMI Embeddings Capture Clinical Concept Relations}
\label{sec:relation_pred}
To assess whether SMI embeddings capture clinical relations in medical ontologies, we evaluated multi-hop relation prediction within the ICD and ATC hierarchies by predicting ancestor–descendant links  within each hierarchy. We also evaluated a cross-ontology 1-hop relation prediction between treatments and their corresponding diagnoses (ICD$\rightarrow$ATC). The experimental setup is provided in Appendix~\ref{app:relation_prediction}. 

We compared SMI with three groups of baselines: (1) the hierarchy-agnostic LM backbone BioBERT~\citep{lee2020biobert}; (2) LMs with structure knowledge, including SapBERT~\citep{liu2021learning}, HiTs \citep{he2024language},  OnT~\citep{OnT}, and BERGAMOT~\citep{sakhovskiy-et-al-2024-bergamot}; and (3) structure-aware embedding baselines, including RotatE~\citep{sun2019rotate} and cui2vec~\citep{beam2020clinical}. Implementation details are provided in Appendix~\ref{app: baselines}.

As shown in Table \ref{fig:hier_rel_pred}, SMI and SMI-W achieve the best performance across all three relation types, outperforming all baseline groups. For within-ontology prediction, SMI-W obtains the highest F-scores on ICD (71.5\%) and ATC (77.5\%), substantially improving over the strongest baseline SapBERT on ICD (67.1\%) and ATC (74.5\%). This indicates that weighting hierarchy levels helps capture parent-child relationships more effectively. Cross-ontology prediction (ICD$\rightarrow$ATC) is challenging because methods trained within each ontology may not directly align diagnosis and medication semantics. Despite this, SMI-W achieves the best Precision (53.6\%), Recall (62.2\%), and F-score (57.6\%), outperforming the strongest baseline HiTs with F-score 53.2\%. These results suggest that aggregating frozen biomedical LM embeddings along ontology-grounded Semantic IDs improves medical concept representations for both within-ontology and cross-ontology relation prediction tasks.
% Among hierarchy-agnostic LMs, the encoder-only BioBERT consistently outperforms the decoder-only BioMistral, consistent with prior findings \citep{Chen2025-cv}. This supports our choice of an encoder backbone for SMI. 

\subsection{SMI Embeddings Capture Medical Ontology Structure}

We applied the Uniform Manifold Approximation and Projection (UMAP) to visualize ICD diagnosis embeddings in 2D \citep{mcinnes2018umap-software}. Each point represents a decimal-level ICD code with frequency $\geq 5$ and is colored by its major disease category. As shown in Fig.~\ref{fig:emb_vis}, BioBERT embeddings are highly mixed across major categories, indicating that hierarchy-agnostic LM embeddings do not clearly separate disease groups. In contrast, SMI embeddings form more distinct clusters by major disease category and further split into coherent subgroups that align with subcategories in the medical ontology. This comparison suggests that hierarchy-aware aggregation helps SMI capture clinically meaningful structure that is not evident from hierarchy-agnostic LM embeddings.

% To assess whether SMI-H yields hierarchy-consistent geometry, we visualize embeddings of all descendant concepts under major category 3 (Endocrine disease category) using a Poincaré map (Fig.~\ref{fig:emb_vis}b) \citep{poincare_map}. A high-resolution visualization is available in Appendix~\ref{app:visual}. The visualization shows a  radial structure: coarse-grained concepts (e.g., 3 Endocrine disease and 3.2 Diabetes mellitus without complication) lie near the center, whereas fine-grained ICD codes lie nearer the boundary.
% This structure aligns with the geometry of the Poincaré ball, where distance grows exponentially with depth. This enables SMI-H to allocate greater capacity to fine-grained concepts while preserving the hierarchical structure of the medical ontology.

To quantify whether the embeddings align with the ICD ontology, we applied K-means clustering with the number of clusters set to the number of major disease categories. As shown in Table~\ref{tab:cluster_aggregate}, SMI achieves stronger ontology-consistent clustering than BioBERT, with higher NMI (0.77 vs. 0.39), ARI (0.56 vs. 0.25), and Purity (0.78 vs. 0.51). SMI also improves per-category clustering quality, increasing Recall from 0.41 to 0.80 and Precision from 0.35 to 0.61, indicating more coherent grouping within the same major category. In addition, SMI better separates ICD major categories, with a higher Silhouette score (0.52 vs. 0.04) and a lower Intra/Inter Ratio (0.29 vs. 0.82). Full aggregate and per-category clustering results with detailed analyses are provided in Appendix~\ref{app:clustering}.

\begin{figure}[t]
    \centering
    \includegraphics[width=\textwidth]{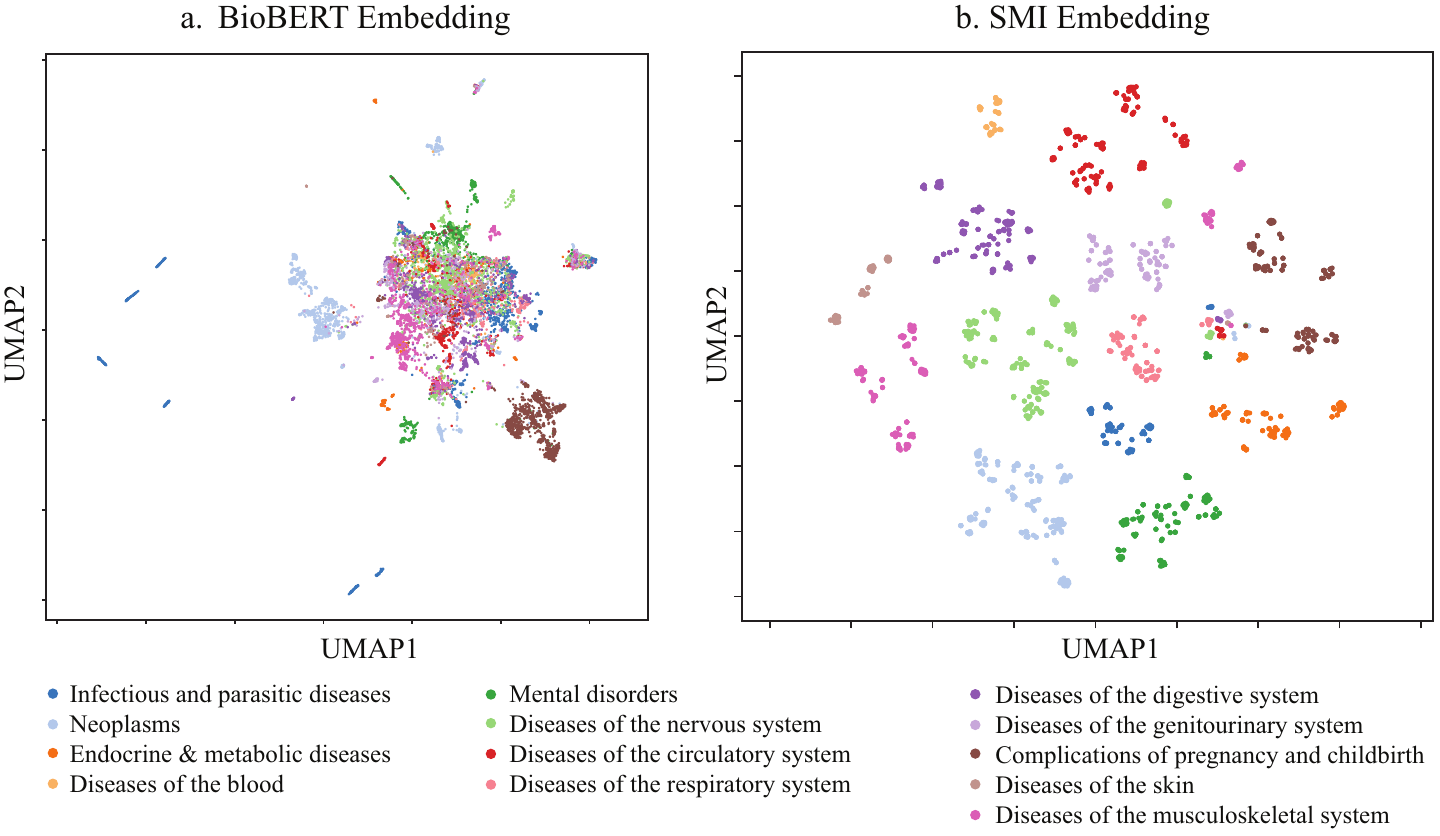} 
    \vspace{-1\baselineskip} 
    \caption{Visualization of ICD diagnosis embeddings across 13 major categories using UMAP. (a) BioBERT embeddings. (b) SMI embeddings. ICD codes embedded by SMI form more distinct category-level clusters that are better aligned with the medical ontology.}
    \label{fig:emb_vis}
    \vspace{-1\baselineskip} 
\end{figure}

\subsection{Evaluation on MIMIC-IV and eICU EHR Datasets}

We evaluated SMI medical concept representations on two EHR datasets, MIMIC-IV and eICU \cite{johnson2023mimic, eicu}. As the first EHR dataset, we used MIMIC‑IV  for clinical prediction tasks, including mortality and readmission prediction as well as drug recommendation. We extracted clinical codes from 123,488 patients and mapped them to their corresponding five-level Semantic IDs.
As the second EHR dataset, we used the eICU dataset to evaluate distribution shifts in patient observations across hospitals using both SMI and BioBERT embedding methods. We mapped diagnoses and medications to 72 unique integer-level ICD-9 codes and 241 unique five-level ATC codes. We selected the 10 hospitals with the largest number of patients, resulting in a total of 24,179 patients and an average of 2,418 patients per hospital. We split each dataset into training, validation, and test sets using a 70\%/10\%/20\% split.
Detailed settings for the MIMIC-IV clinical prediction tasks and the eICU distribution shift analysis are provided in Sections~\ref{sec: pred_results} and~\ref{sec:dist_shift_main}, respectively.

\subsection{SMI Improves Clinical Prediction and Rare Drug Prediction}
\label{sec: pred_results}
For clinical prediction, we extracted diagnosis, procedure, and medication codes from MIMIC-IV and evaluated three common tasks. Mortality prediction assesses whether a patient has died within 90 days after discharge. Readmission prediction assesses whether a patient is readmitted within the next 15 days after discharge. Drug recommendation predicts the set of medications for each visit given the patient’s prior clinical history. Mortality and readmission are binary classification tasks evaluated with Area Under Precision-Recall Curve (AUPRC) and Area Under Receiver Operating Characteristic Curve (AUROC). Drug recommendation is a multi-label prediction task evaluated with AUPRC, F-score, and Jaccard. We also predicted rare drugs using 98 ATC codes with frequency between 10 and 20, using Recall of top 15 predictions as a metric. We mapped clinical codes to their SMI embeddings and used mean pooling over tokens for a patient-level representation. We compared with two groups of baselines: (1) clinical concept embedding methods such as GRAM \citep{choi2017gram}, cui2vec \citep{beam2020clinical}, and RotatE \citep{sun2019rotate}; (2) biomedical LMs including BioBERT \citep{lee2020biobert}, SapBERT \citep{liu2021learning}, G-BERT \citep{shang2019pre}, BEHRT \citep{li2020behrt}, ClinicalBERT \citep{alsentzer2019publicly}, BoxLM~\citep{tan2025boxlm}, and BERGAMOT~\citep{sakhovskiy-et-al-2024-bergamot}. We provided implementation details in Appendix~\ref{app: baselines}.

% We made classification using a linear probing approach with a single linear layer. For binary prediction of mortality and readmission, a linear classifier with a weight $w\in R^{d\times 1}$ and a sigmoid activation was trained using binary cross-entropy loss. All models were implemented in PyTorch and optimized with Adam optimizer 
% (learning rate $1\times10^{-4}$, weight decay  $1\times10^{-5}$). Training used a batch size of 64 for up to 20 epochs with early stopping. An 80/20 train–test split was applied, and performance was evaluated using AUROC and AUPRC.

% \paragraph{SMI Helps Discover Rare Clinical Code}
% \todo[inline]{TO-DO}

\begin{table}[t]
\centering
\footnotesize
\caption{Clinical prediction results on MIMIC-IV for mortality, readmission, and (rare) drug recommendation, comparing  SMI with two groups of baselines. Metrics are reported as mean (standard error) from bootstrap. The best results in each column are highlighted in bold.}
\label{tab:tasks_mimic4}
\begin{tabular}{lcccccccc}
\toprule
\textbf{Method} &
\multicolumn{2}{c}{\textbf{Mortality}} &
\multicolumn{2}{c}{\textbf{Readmission}} &
\multicolumn{3}{c}{\textbf{Drug}} &
\multicolumn{1}{c}{\textbf{Rare Drug}} \\
\cmidrule(lr){2-3}\cmidrule(lr){4-5}\cmidrule(lr){6-8}\cmidrule(lr){9-9}
& \textbf{AUPRC} & \textbf{AUROC}
& \textbf{AUPRC} & \textbf{AUROC}
& \textbf{AUPRC} & \textbf{F-score} & \textbf{Jaccard}
& \textbf{Recall@15} \\
\midrule
\rowcolor{methodgreen}
\textbf{SMI} & 44.5 (0.5) & 87.4 (1.0) & \textbf{48.9 (0.2)} & 72.5 (0.5) & 75.6 (0.1) & 62.4 (0.1) & 48.6 (0.1) & 41.4 (0.2) \\
\rowcolor{methodgreen}
\textbf{SMI-W} & \textbf{45.1 (0.3)} & \textbf{88.0 (0.7)} & 48.5 (0.3) & \textbf{72.9 (0.3)} & \textbf{77.6 (0.1)} & \textbf{65.3 (0.2)} & \textbf{50.5 (0.1)} & \textbf{44.2 (0.3)}\\
% 41.3 (0.8) & 83.7 (1.3) & 45.7 (0.6) & 70.3 (0.6) & \textbf{76.3 (0.1)} & \textbf{64.3 (0.2)} & \textbf{50.7 (0.2)} & \textbf{44.2 (0.3)} \\
\midrule
GRAM & 38.5 (0.7) & 79.1 (0.9) & 43.2 (0.8) & 68.2 (0.4) & 72.0 (0.3) & 61.3 (0.4) & 48.6 (0.4) & 34.5 (0.2) \\
cui2vec  & 39.0 (0.4) & 80.4 (0.8) & 43.7 (0.6) & 68.9 (0.3) & 75.2 (0.1) & 64.3 (0.3) & 49.9  (0.4) & 35.6 (0.1) \\
RotatE  & 39.2 (0.3) & 79.6 (0.6) & 43.6 (0.4) & 69.2 (0.6) & 74.1 (0.2) & 62.7 (0.4) & 48.7 (0.4) & 36.4 (0.1) \\
HiTs & 42.7 (0.6) & 81.5 (0.7) & 45.2 (0.4) & 69.6 (0.8) & 74.8 (0.2)& 61.6 (0.2) & 48.0 (0.1) & 37.5 (0.1) \\
OnT & 43.3 (0.5) & 83.4 (1.1) & 45.9 (0.7) &70.3 (0.8)  & 71.0 (0.1) & 57.3 (0.2) & 44.5 (0.1) &35.8 (0.1) \\
BioBERT & 43.2 (0.9) & 86.1 (1.3) & 46.5 (0.9) & 71.5 (1.2) & 73.9 (0.1)& 60.3 (0.2) & 44.5 (0.2) & 37.6 (0.2) \\
SapBERT  & 42.4 (0.6) & 83.7 (1.2) & 45.3 (0.7) & 70.2 (1.1)  & 74.6 (0.1) & 60.5 (0.2) & 45.7(0.4) & 40.0 (0.3) \\
G-BERT & 42.7 (0.4) & 85.0 (1.3) & 47.3 (0.4) & 70.1 (0.7) & 75.4 (0.1)& 61.4 (0.3) & 45.5 (0.3) & 40.7 (0.3)  \\
BEHRT  & 42.0 (0.3) & 84.3 (0.7) & 44.5 (0.4) &  69.0 (0.8) & 74.3 (0.2) & 60.2 (0.2) & 45.7 (0.3) & 37.5 (0.3)  \\
ClinicalBERT & 43.5 (0.9) & 86.2 (0.8) & 45.7 (0.6) & 70.4 (0.6) & 73.8 (0.2) & 59.4(0.4) & 44.0 (0.4) & 38.4 (0.4) \\
BoxLM & 41.7 (1.4) & 83.5 (1.2) & 43.8 (0.8) & 69.1 (0.9) & 73.6 (0.4) & 59.1 (0.2) & 43.3 (0.5) & 35.9 (0.5)\\
BERGAMOT & 43.8 (0.9) & 85.9 (0.7) & 45.6 (0.6) & 71.0 (1.1) & 75.2 (0.3) & 61.7 (0.6) & 48.1 (0.3) & 38.6 (0.5)  \\
\bottomrule
\end{tabular}
\vspace{-0.5\baselineskip} 
\end{table}

As shown in Table~\ref{tab:tasks_mimic4}, SMI and SMI-W consistently improve clinical prediction on MIMIC-IV. SMI and SMI-W generally achieve stronger mortality and readmission prediction performance than the baselines in terms of AUPRC and AUROC. The gains are especially clear for drug recommendation, where SMI-W achieves the best AUPRC, F-score, and Jaccard scores (77.6\%, 65.3\%, and 50.5\%), significantly outperforming the strongest baseline ($p=0.015$). For rare drug prediction, SMI-W achieves the highest Recall@15 (44.2\%), significantly outperforming the strongest baseline G-BERT ($p=0.0081$), showing that SMI is particularly effective for long-tailed fine-grained codes. These results suggest that aggregating frozen biomedical LM embeddings along ontology-grounded Semantic IDs benefits downstream clinical prediction, particularly rare drug prediction.
Compared with GRAM, which performs hierarchy-aware aggregation within a supervised prediction model, SMI constructs hierarchy-aware medical concept representations from frozen biomedical LM embeddings without task-specific training. 
Compared with structure-aware encoders such as HiTs and OnT, which enforce structural consistency for individual concept embeddings, SMI explicitly aggregates hierarchy-level semantics into a final concept representation.

% \begin{figure*}[!t]
%     \centering
%     \includegraphics[width=\textwidth]{figures/dist_shifts.png} 
%     \caption{Cross-hospital distribution shift analysis based on average absolute log-likelihood differences ($|\Delta|$) between source and target hospitals. Each heatmap illustrates the distribution shift of patient embedding using SMI (left) and BioBERT  (right), both evaluated with a VAE-based density model. SMI embeddings exhibit smaller distribution shifts (average $|\Delta| = 50.77$) compared to  BioBERT ($|\Delta| = 74.06$). 
%     % suggesting better cross-site consistency and generalization.
%     }
%     \label{fig:dist_shifts}
% \end{figure*}

% \begin{figure}[t]
%     \centering
%     \includegraphics[width=\linewidth]{figures/dist_shift_diff_heatmap.png}
%     \caption{Cross-hospital distribution shift comparison between BioBERT and SMI. Each cell shows $|\Delta|_{\text{BioBERT}} - |\Delta|_{\text{SMI}}$; red colors indicate larger shifts under BioBERT, implying that SMI is more cross-site robust. The p-value is computed on off-diagonal hospital pairs and tests whether $|\Delta|_{\text{SMI}} < |\Delta|_{\text{BioBERT}}$, indicating that SMI achieves significantly lower cross-hospital shift.}
%     \label{fig:dist_shift_diff}
% \end{figure}

\begin{table}[t]
\centering
\footnotesize
\setlength{\tabcolsep}{4.5pt}
\caption{Ablations on clinical relation prediction comparing the LM baseline, SMI, and SMI-W.}
\label{tab:ablation_hier_pred}
\begin{tabular}{lccc ccc ccc}
\toprule
\multirow{2}{*}{\textbf{Method}} &
\multicolumn{3}{c}{\textbf{ICD}} &
\multicolumn{3}{c}{\textbf{ATC}} &
\multicolumn{3}{c}{\textbf{ICD$\rightarrow$ATC}} \\
\cmidrule(lr){2-4}\cmidrule(lr){5-7}\cmidrule(lr){8-10}
& \textbf{Precision} & \textbf{Recall} & \textbf{F-score}
& \textbf{Precision} & \textbf{Recall} & \textbf{F-score}
& \textbf{Precision} & \textbf{Recall} & \textbf{F-score} \\
\midrule
% \rowcolor{baselinegray}
BioBERT & 30.4 & 75.8 & 43.4 & 50.7 & 45.7 & 48.1 & 49.3 & 57.0 & 52.9 \\
\rowcolor{methodgreen}
SMI & 70.6 & 67.9 & 69.2 & 83.8 & 67.6 & 74.9 & 53.2 & 61.4 & 57.0 \\
\rowcolor{methodgreen}
SMI-W & 77.2 & 66.6 & 71.5 & 86.1 & 70.5 & 77.5 & 53.6 & 62.2 & 57.6 \\
\bottomrule
\end{tabular}
\vspace{-0.5\baselineskip} 
\end{table}

\subsection{SMI Mitigates Cross-Hospital Distribution Shift}
\label{sec:dist_shift_main}

The multi-center eICU dataset exhibits substantial heterogeneity across hospitals in patient populations, clinical practice patterns, and data collection processes \citep{zhu2025fedweight}. We thus use \textit{hospital} as the patient-group axis and evaluate cross-site distribution shift, as such heterogeneity can degrade the generalizability of learned patient representations. Specifically, we fit a hospital-specific VAE density model on patient embeddings encoded by either SMI or BioBERT \citep{vae}. For a source hospital $k$ and target hospital $t$, we measured the shift by the log-likelihood gap:
\begin{align}
    \Delta = \mathbb{E}_{x_k} \left[ \log p_k(x_k) \right] - \mathbb{E}_{x_t} \left[ \log p_k(x_t) \right]
\end{align}
\clearpage
\begin{wrapfigure}{r}{0.5\textwidth}
    \vspace{-10pt}
    \centering
    \includegraphics[width=\linewidth]{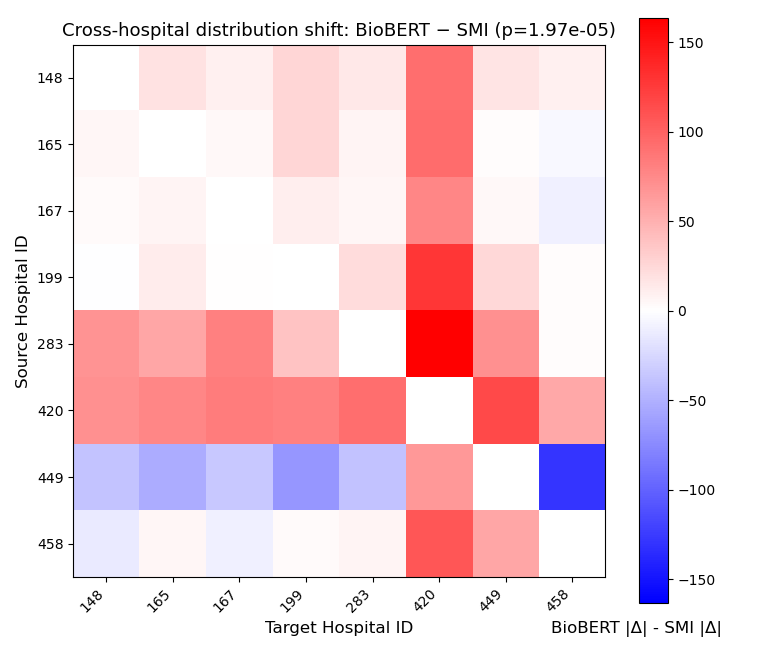}
    \vspace{-10pt}
    \caption{Cross-hospital distribution shift comparison between BioBERT and SMI. Each cell shows $|\Delta|_{\text{BioBERT}} - |\Delta|_{\text{SMI}}$; red colors indicate that SMI is more cross-site robust. P-value indicates that SMI achieves significantly lower distribution shift.}
    \label{fig:dist_shift_diff}
\end{wrapfigure}
where $\log p_k(x_k)$ and $\log p_k(x_t)$ are the log‑likelihoods of in‑domain $x_k$ and out‑of‑domain $x_t$ evaluated by the VAE trained on hospital $k$.  More evaluation details are provided in Appendix~\ref{app:dist_shift}. Fig.~\ref{fig:dist_shifts} visualizes the cross-hospital distribution shift $|\Delta|$ for each source–target hospital pair under a source-trained density model, where larger $|\Delta|$ indicates greater shift between the source and target distributions. SMI yields consistently smaller shift  than BioBERT (50.77 vs.\ 74.06), indicating improved cross-site robustness. Fig.~\ref{fig:dist_shift_diff} shows cross-hospital distribution shift between BioBERT and SMI by the difference ($|\Delta{\text{BioBERT}} - |\Delta|{\text{SMI}}$) across source–target hospital pairs. Most off-diagonal entries are red, indicating larger shifts under BioBERT and thus better cross-site robustness for SMI. We also computed p-values using the Wilcoxon signed-rank paired test over all off-diagonal hospital pairs. SMI achieves significantly smaller cross-hospital shifts ($p=1.97\times 10^{-5}$), showing that SMI improves cross-site robustness of representations.

\subsection{Ablation Studies}

We ablate SMI on clinical relation prediction by comparing BioBERT, SMI, and SMI-W (Table~\ref{tab:ablation_hier_pred}). BioBERT directly encodes each code description without using ontology structure. As shown in Table~\ref{tab:ablation_hier_pred}, SMI consistently outperforms BioBERT across ICD, ATC, and ICD$\rightarrow$ATC relation prediction, showing that simple aggregation of ontology path improves hierarchy-aware concept representation. SMI-W further improves over uniform SMI by selecting global level weights, suggesting that hierarchy levels contribute unequally to the final concept representation.

\section{Conclusion}

Our study makes the following key contributions. First, we propose SMI, a training-free framework that generates hierarchy-aware concept representations by integrating medical ontologies with LM embeddings. Second, we provide a theoretical analysis showing that SMI improves SNR and yields more informative embeddings. Third, qualitative analyses confirm that SMI learns clinically meaningful structures aligned with medical hierarchies. Finally, extensive experiments on clinical relation prediction, downstream clinical prediction, and cross-site distribution shift demonstrate the broad utility of SMI. The current study focuses on expert-curated coding systems such as ICD and ATC. Applying SMI to emerging disease classifications, newly introduced drugs, or local coding standards may require additional ontology validation and code mapping. For example, Canadian healthcare systems use ICD-10-CA or Drug Identification Numbers, which must be carefully mapped before deployment. Future work will extend SMI to clinician-validated ontology harmonization across coding systems, missing-link discovery, and broader biomedical knowledge resources. We will also study how to jointly integrate multiple ontologies within a unified SMI framework.

 % Future work will extend SMI to clinician-validated ontology harmonization, missing-link discovery, and broader biomedical knowledge resources such as UMLS. We will also study how to integrate multiple ontologies within a unified SMI framework.

% Clinician oversight is therefore important to ensure that the selected ontology and code mappings match the target healthcare system. 

\section*{Acknowledgments}

This work was supported in part by the NVIDIA Academic Grant Program through an award to Ziyang Song at Ohio University. 

\clearpage

% \section*{Acknowledgments}

% This work was supported in part by the NVIDIA Academic Grant Program through an award to Ziyang Song at Ohio University. 

\bibliographystyle{unsrtnat}
\bibliography{example_paper}

@article{lee2020biobert,
  title={BioBERT: a pre-trained biomedical language representation model for biomedical text mining},
  author={Lee, Jinhyuk and Yoon, Wonjin and Kim, Sungdong and Kim, Donghyeon and Kim, Sunkyu and So, Chan Ho and Kang, Jaewoo},
  journal={Bioinformatics},
  volume={36},
  number={4},
  pages={1234--1240},
  year={2020},
  publisher={Oxford University Press}
}

@ARTICLE{icd,
  title     = "International classification of diseases, 10th edition, clinical
               modification and procedure coding system: descriptive overview
               of the next generation {HIPAA} code sets",
  author    = "Steindel, Steven J",
  journal   = "J. Am. Med. Inform. Assoc.",
  publisher = "Oxford University Press (OUP)",
  volume    =  17,
  number    =  3,
  pages     = "274--282",
  month     =  may,
  year      =  2010,
  language  = "en"
}

@article{atc,
    doi = {10.1371/journal.pone.0035254},
    author = {Chen, Lei AND Zeng, Wei-Ming AND Cai, Yu-Dong AND Feng, Kai-Yan AND Chou, Kuo-Chen},
    journal = {PLOS ONE},
    publisher = {Public Library of Science},
    title = {Predicting Anatomical Therapeutic Chemical (ATC) Classification of Drugs by Integrating Chemical-Chemical Interactions and Similarities},
    year = {2012},
    month = {04},
    volume = {7},
    url = {https://doi.org/10.1371/journal.pone.0035254},
    pages = {1-7},
    number = {4},

}

@misc{saferxagent,
      title={SafeRx-Agent: A Knowledge-Grounded Multi-Agent Framework for Safe and Explainable Medication Recommendation}, 
      author={Xinyu Wang and Hanwei Wu and Zhenghan Tai and Sicheng Lyu and Qincheng Lu and Ziyu Zhao and Jijun Chi and Jingrui Tian and Xiao-Wen Chang and Ziyang Song},
      year={2026},
      eprint={2605.29146},
      archivePrefix={arXiv},
      primaryClass={cs.CL},
      url={https://arxiv.org/abs/2605.29146}, 
}

@inproceedings{timelygpt,
author = {Song, Ziyang and Lu, Qincheng and Xu, Hao and Zhu, He and Buckeridge, David and Li, Yue},
title = {TimelyGPT: Extrapolatable Transformer Pre-training for Long-term Time-Series Forecasting in Healthcare},
year = {2024},
isbn = {9798400713026},
publisher = {Association for Computing Machinery},
address = {New York, NY, USA},
url = {https://doi.org/10.1145/3698587.3701364},
doi = {10.1145/3698587.3701364},
booktitle = {Proceedings of the 15th ACM International Conference on Bioinformatics, Computational Biology and Health Informatics},
articleno = {16},
numpages = {10},
location = {Shenzhen, China},
series = {BCB '24}
}

@ARTICLE{trajgpt,
  title     = "{TrajGPT}: Irregular time-series representation learning of
               health trajectory",
  author    = "Song, Ziyang and Lu, Qincheng and Zhu, He and Buckeridge, David
               and Li, Yue",
  journal   = "IEEE J. Biomed. Health Inform.",
  publisher = "Institute of Electrical and Electronics Engineers (IEEE)",
  volume    = "PP",
  pages     = "1--14",
  month     =  oct,
  year      =  2025,
  copyright = "https://ieeexplore.ieee.org/Xplorehelp/downloads/license-information/IEEE.html",
  language  = "en"
}

@article{eicu,
author = {Pollard, Tom and Johnson, Alistair and Raffa, Jesse and Celi, Leo and Mark, Roger and Badawi, Omar},
year = {2018},
month = {09},
pages = {180178},
title = {The eICU Collaborative Research Database, a freely available multi-center database for critical care research},
volume = {5},
journal = {Scientific Data},
doi = {10.1038/sdata.2018.178}
}

@article{vae,
  title={Auto-encoding variational bayes},
  author={Kingma, Diederik P and Welling, Max},
  journal={arXiv preprint arXiv:1312.6114},
  year={2013}
}

@article{burda2015importance,
  title={Importance weighted autoencoders},
  author={Burda, Yuri and Grosse, Roger and Salakhutdinov, Ruslan},
  journal={arXiv preprint arXiv:1509.00519},
  year={2015}
}

@article{mcinnes2018umap-software,
  title={UMAP: Uniform Manifold Approximation and Projection},
  author={McInnes, Leland and Healy, John and Saul, Nathaniel and Grossberger, Lukas},
  journal={The Journal of Open Source Software},
  volume={3},
  number={29},
  pages={861},
  year={2018}
}

@inproceedings{skip_gram,
author = {Mikolov, Tomas and Sutskever, Ilya and Chen, Kai and Corrado, Greg and Dean, Jeffrey},
title = {Distributed representations of words and phrases and their compositionality},
year = {2013},
publisher = {Curran Associates Inc.},
address = {Red Hook, NY, USA},
booktitle = {Proceedings of the 27th International Conference on Neural Information Processing Systems - Volume 2},
pages = {3111–3119},
numpages = {9},
location = {Lake Tahoe, Nevada},
series = {NIPS'13}
}

@article{zhu2025fedweight,
  title={FedWeight: mitigating covariate shift of federated learning on electronic health records data through patients re-weighting},
  author={Zhu, He and Bai, Jun and Li, Na and Li, Xiaoxiao and Liu, Dianbo and Buckeridge, David L and Li, Yue},
  journal={npj Digital Medicine},
  volume={8},
  number={1},
  pages={286},
  year={2025},
  publisher={Nature Publishing Group UK London}
}

@InProceedings{pmlr-v80-ganea18a,
  title = 	 {Hyperbolic Entailment Cones for Learning Hierarchical Embeddings},
  author =       {Ganea, Octavian and Becigneul, Gary and Hofmann, Thomas},
  booktitle = 	 {Proceedings of the 35th International Conference on Machine Learning},
  pages = 	 {1646--1655},
  year = 	 {2018},
  editor = 	 {Dy, Jennifer and Krause, Andreas},
  volume = 	 {80},
  series = 	 {Proceedings of Machine Learning Research},
  month = 	 {10--15 Jul},
  publisher =    {PMLR},
  url = 	 {https://proceedings.mlr.press/v80/ganea18a.html}
}

@ARTICLE{Zou2022-te,
  title    = "Modeling electronic health record data using an end-to-end
              knowledge-graph-informed topic model",
  author   = "Zou, Yuesong and Pesaranghader, Ahmad and Song, Ziyang and Verma,
              Aman and Buckeridge, David L and Li, Yue",
  journal  = "Sci. Rep.",
  volume   =  12,
  number   =  1,
  pages    = "17868",
  month    =  oct,
  year     =  2022,
  language = "en"
}

@article{he2024language,
  title={Language models as hierarchy encoders},
  author={He, Yuan and Yuan, Moy and Chen, Jiaoyan and Horrocks, Ian},
  journal={Advances in Neural Information Processing Systems},
  volume={37},
  pages={14690--14711},
  year={2024}
}

@article{johnson2023mimic,
  title={MIMIC-IV, a freely accessible electronic health record dataset},
  author={Johnson, Alistair EW and Bulgarelli, Lucas and Shen, Lu and Gayles, Alvin and Shammout, Ayad and Horng, Steven and Pollard, Tom J and Hao, Sicheng and Moody, Benjamin and Gow, Brian and others},
  journal={Scientific data},
  volume={10},
  number={1},
  pages={1},
  year={2023},
  publisher={Nature Publishing Group UK London}
}

@article{zhang_2019,
	author = {Zhang, Yijia and Chen, Qingyu and Yang, Zhihao and Lin, Hongfei and Lu, Zhiyong},
	journal = {Scientific Data},
	title = {{BioWordVec, improving biomedical word embeddings with subword information and MeSH}},
	year = {2019}
}

@software{kury_atcd_2024,
  author       = {Kury, Fabr{\'\i}cio},
  title        = {Scrape Anatomical-Therapeutic-Chemical (ATC) classes from the WHO Collaborating Centre for Drug Statistics Methodology website},
  year         = {2024},
  publisher    = {GitHub},
  url          = {https://github.com/fabkury/atcd}
}

@inproceedings{remy-etal-2022-biolord,
    title = "{B}io{LORD}: Learning Ontological Representations from Definitions for Biomedical Concepts and their Textual Descriptions",
    author = "Remy, Fran{\c{c}}ois  and
      Demuynck, Kris  and
      Demeester, Thomas",
    editor = "Goldberg, Yoav  and
      Kozareva, Zornitsa  and
      Zhang, Yue",
    booktitle = "Findings of the Association for Computational Linguistics: EMNLP 2022",
    month = dec,
    year = "2022",
    address = "Abu Dhabi, United Arab Emirates",
    publisher = "Association for Computational Linguistics",
    url = "https://aclanthology.org/2022.findings-emnlp.104/",
    doi = "10.18653/v1/2022.findings-emnlp.104",
    pages = "1454--1465"
}

@ARTICLE{Yuan2022-fj,
  title     = "{CODER}: Knowledge-infused cross-lingual medical term embedding
               for term normalization",
  author    = "Yuan, Zheng and Zhao, Zhengyun and Sun, Haixia and Li, Jiao and
               Wang, Fei and Yu, Sheng",
  journal   = "J. Biomed. Inform.",
  publisher = "Elsevier BV",
  volume    =  126,
  number    =  103983,
  pages     = "103983",
  month     =  feb,
  year      =  2022,
  copyright = "http://www.elsevier.com/open-access/userlicense/1.0/",
  language  = "en"
}

@inproceedings{
gao2018representation,
title={Representation Degeneration Problem in Training Natural Language Generation Models},
author={Jun Gao and Di He and Xu Tan and Tao Qin and Liwei Wang and Tieyan Liu},
booktitle={International Conference on Learning Representations},
year={2019},
url={https://openreview.net/forum?id=SkEYojRqtm},
}

@inproceedings{choi2017gram,
  title={GRAM: graph-based attention model for healthcare representation learning},
  author={Choi, Edward and Bahadori, Mohammad Taha and Song, Le and Stewart, Walter F and Sun, Jimeng},
  booktitle={Proceedings of the 23rd ACM SIGKDD international conference on knowledge discovery and data mining},
  pages={787--795},
  year={2017}
}

@article{shang2019pre,
  title={Pre-training of graph augmented transformers for medication recommendation},
  author={Shang, Junyuan and Ma, Tengfei and Xiao, Cao and Sun, Jimeng},
  journal={arXiv preprint arXiv:1906.00346},
  year={2019}
}

@inproceedings{beam2020clinical,
  title={Clinical concept embeddings learned from massive sources of multimodal medical data},
  author={Beam, Andrew L and Kompa, Benjamin and Schmaltz, Allen and Fried, Inbar and Weber, Griffin and Palmer, Nathan and Shi, Xu and Cai, Tianxi and Kohane, Isaac S},
  booktitle={Pacific Symposium on Biocomputing. Pacific Symposium on Biocomputing},
  volume={25},
  pages={295},
  year={2020}
}

@article{sun2019rotate,
  title={Rotate: Knowledge graph embedding by relational rotation in complex space},
  author={Sun, Zhiqing and Deng, Zhi-Hong and Nie, Jian-Yun and Tang, Jian},
  journal={arXiv preprint arXiv:1902.10197},
  year={2019}
}

@article{li2020behrt,
  title={BEHRT: transformer for electronic health records},
  author={Li, Yikuan and Rao, Shishir and Solares, Jos{\'e} Roberto Ayala and Hassaine, Abdelaali and Ramakrishnan, Rema and Canoy, Dexter and Zhu, Yajie and Rahimi, Kazem and Salimi-Khorshidi, Gholamreza},
  journal={Scientific reports},
  volume={10},
  number={1},
  pages={7155},
  year={2020},
  publisher={Nature Publishing Group UK London}
}

@article{rasmy2021med,
  title={Med-BERT: pretrained contextualized embeddings on large-scale structured electronic health records for disease prediction},
  author={Rasmy, Laila and Xiang, Yang and Xie, Ziqian and Tao, Cui and Zhi, Degui},
  journal={NPJ digital medicine},
  volume={4},
  number={1},
  pages={86},
  year={2021},
  publisher={Nature Publishing Group UK London}
}

@inproceedings{alsentzer2019publicly,
  title={Publicly available clinical BERT embeddings},
  author={Alsentzer, Emily and Murphy, John and Boag, William and Weng, Wei-Hung and Jindi, Di and Naumann, Tristan and McDermott, Matthew},
  booktitle={Proceedings of the 2nd clinical natural language processing workshop},
  pages={72--78},
  year={2019}
}

@inproceedings{sung2020biomedical,
    title = "Biomedical Entity Representations with Synonym Marginalization",
    author = "Sung, Mujeen  and
      Jeon, Hwisang  and
      Lee, Jinhyuk  and
      Kang, Jaewoo",
    editor = "Jurafsky, Dan  and
      Chai, Joyce  and
      Schluter, Natalie  and
      Tetreault, Joel",
    booktitle = "Proceedings of the 58th Annual Meeting of the Association for Computational Linguistics",
    month = jul,
    year = "2020",
    address = "Online",
    publisher = "Association for Computational Linguistics",
    url = "https://aclanthology.org/2020.acl-main.335/",
    doi = "10.18653/v1/2020.acl-main.335",
    pages = "3641--3650"
}

@inproceedings{liu2021learning,
	title={Learning Domain-Specialised Representations for Cross-Lingual Biomedical Entity Linking},
	author={Liu, Fangyu and Vuli{\'c}, Ivan and Korhonen, Anna and Collier, Nigel},
	booktitle={Proceedings of ACL-IJCNLP 2021},
	pages = {565--574},
	month = aug,
	year={2021}
}

@inproceedings{sakhovskiy-et-al-2024-bergamot,
    title = "Biomedical Entity Representation with Graph-Augmented Multi-Objective Transformer",
    author = "Sakhovskiy, Andrey and Semenova, Natalia and Kadurin, Artur and Tutubalina, Elena",
    booktitle = "Findings of the Association for Computational Linguistics: NAACL 2024",
    month = jun,
    year = "2024",
    address = "Mexico City, Mexico",
    publisher = "Association for Computational Linguistics",
  }

@ARTICLE{Remy2024-hb,
  title     = "{BioLORD-2023}: semantic textual representations fusing large
               language models and clinical knowledge graph insights",
  author    = "Remy, Fran{\c c}ois and Demuynck, Kris and Demeester, Thomas",
  journal   = "J. Am. Med. Inform. Assoc.",
  publisher = "Oxford University Press (OUP)",
  volume    =  31,
  number    =  9,
  pages     = "1844--1855",
  month     =  sep,
  year      =  2024,
  copyright = "https://creativecommons.org/licenses/by/4.0/",
  language  = "en"
}

@inproceedings{OnT,
author = {Yang, Hui and Chen, Jiaoyan and He, Yuan and Gao, Yongsheng and Horrocks, Ian},
title = {Language Models as Ontology Encoders},
year = {2025},
isbn = {978-3-032-09526-8},
publisher = {Springer-Verlag},
address = {Berlin, Heidelberg},
url = {https://doi.org/10.1007/978-3-032-09527-5_24},
doi = {10.1007/978-3-032-09527-5_24},
booktitle = {The Semantic Web – ISWC 2025: 24th International Semantic Web Conference, Nara, Japan, November 2–6, 2025, Proceedings, Part I},
pages = {443–461},
numpages = {19},
location = {Nara, Japan}
}

@inproceedings{
tan2025boxlm,
title={Box{LM}: Unifying Structures and Semantics of Medical Concepts for Diagnosis Prediction in Healthcare},
author={Yanchao Tan and Hang Lv and Yunfei Zhan and Guofang Ma and Bo Xiong and Carl Yang},
booktitle={Forty-second International Conference on Machine Learning},
year={2025},
url={https://openreview.net/forum?id=jIci8mGVeh}
}

@article{schick2020rare, 
title={Rare Words: A Major Problem for Contextualized Embeddings and How to Fix it by Attentive Mimicking}, 
volume={34}, url={https://ojs.aaai.org/index.php/AAAI/article/view/6403}, 
DOI={10.1609/aaai.v34i05.6403}, 
abstractNote={&lt;p&gt;Pretraining deep neural network architectures with a language modeling objective has brought large improvements for many natural language processing tasks. Exemplified by BERT, a recently proposed such architecture, we demonstrate that despite being trained on huge amounts of data, deep language models still struggle to understand rare words. To fix this problem, we adapt Attentive Mimicking, a method that was designed to explicitly learn embeddings for rare words, to deep language models. In order to make this possible, we introduce &lt;em&gt;one-token approximation&lt;/em&gt;, a procedure that enables us to use Attentive Mimicking even when the underlying language model uses subword-based tokenization, i.e., it does not assign embeddings to all words. To evaluate our method, we create a novel dataset that tests the ability of language models to capture semantic properties of words without any task-specific fine-tuning. Using this dataset, we show that adding our adapted version of Attentive Mimicking to BERT does substantially improve its understanding of rare words.&lt;/p&gt;}, number={05}, journal={Proceedings of the AAAI Conference on Artificial Intelligence}, 
author={Schick, Timo and Sch{\"u}tze, Hinrich},
year={2020}, 
month={Apr.}, 
pages={8766-8774} }

@inproceedings{schick2020bertram,
    title = "{BERTRAM}: Improved Word Embeddings Have Big Impact on Contextualized Model Performance",
    author={Schick, Timo and Sch{\"u}tze, Hinrich},

    editor = "Jurafsky, Dan  and
      Chai, Joyce  and
      Schluter, Natalie  and
      Tetreault, Joel",
    booktitle = "Proceedings of the 58th Annual Meeting of the Association for Computational Linguistics",
    month = jul,
    year = "2020",
    address = "Online",
    publisher = "Association for Computational Linguistics",
    url = "https://aclanthology.org/2020.acl-main.368/",
    doi = "10.18653/v1/2020.acl-main.368",
    pages = "3996--4007"
}

@inproceedings{
TIGER,
title={Recommender Systems with Generative Retrieval},
author={Shashank Rajput and Nikhil Mehta and Anima Singh and Raghunandan Hulikal Keshavan and Trung Vu and Lukasz Heldt and Lichan Hong and Yi Tay and Vinh Q. Tran and Jonah Samost and Maciej Kula and Ed H. Chi and Maheswaran Sathiamoorthy},
booktitle={Thirty-seventh Conference on Neural Information Processing Systems},
year={2023},
url={https://openreview.net/forum?id=BJ0fQUU32w}
}

@INPROCEEDINGS{LC-Rec,
  title           = "Adapting large language models by integrating
                     collaborative semantics for recommendation",
  booktitle       = "2024 {IEEE} 40th International Conference on Data
                     Engineering ({ICDE})",
  author          = "Zheng, Bowen and Hou, Yupeng and Lu, Hongyu and Chen, Yu
                     and Zhao, Wayne Xin and Chen, Ming and Wen, Ji-Rong",
  publisher       = "IEEE",
  month           =  may,
  year            =  2024,
  conference      = "2024 IEEE 40th International Conference on Data
                     Engineering (ICDE)",
  location        = "Utrecht, Netherlands"
}

@inproceedings{LETTER,
author = {Wang, Wenjie and Bao, Honghui and Lin, Xinyu and Zhang, Jizhi and Li, Yongqi and Feng, Fuli and Ng, See-Kiong and Chua, Tat-Seng},
title = {Learnable Item Tokenization for Generative Recommendation},
year = {2024},
isbn = {9798400704369},
publisher = {Association for Computing Machinery},
address = {New York, NY, USA},
url = {https://doi.org/10.1145/3627673.3679569},
doi = {10.1145/3627673.3679569},
booktitle = {Proceedings of the 33rd ACM International Conference on Information and Knowledge Management},
pages = {2400–2409},
numpages = {10},
location = {Boise, ID, USA},
series = {CIKM '24}
}

@inproceedings{ETEGRec,
author = {Liu, Enze and Zheng, Bowen and Ling, Cheng and Hu, Lantao and Li, Han and Zhao, Wayne Xin},
title = {Generative Recommender with End-to-End Learnable Item Tokenization},
year = {2025},
isbn = {9798400715921},
publisher = {Association for Computing Machinery},
address = {New York, NY, USA},
url = {https://doi.org/10.1145/3726302.3729989},
doi = {10.1145/3726302.3729989},
booktitle = {Proceedings of the 48th International ACM SIGIR Conference on Research and Development in Information Retrieval},
pages = {729–739},
numpages = {11},
location = {Padua, Italy},
series = {SIGIR '25}
}

@inproceedings{song2022seed,
author = {Song, Ziyang and Hu, Yuanyi and Verma, Aman and Buckeridge, David L. and Li, Yue},
title = {Automatic Phenotyping by a Seed-guided Topic Model},
year = {2022},
isbn = {9781450393850},
publisher = {Association for Computing Machinery},
address = {New York, NY, USA},
url = {https://doi.org/10.1145/3534678.3542675},
doi = {10.1145/3534678.3542675},
booktitle = {Proceedings of the 28th ACM SIGKDD Conference on Knowledge Discovery and Data Mining},
pages = {4713–4723},
numpages = {11},
location = {Washington DC, USA},
series = {KDD '22}
}

@ARTICLE{Yang2026-mb,
  title     = "{PheCode-guided} multi-modal topic modeling of electronic health
               records improves disease incidence prediction and {GWAS}
               discovery from {UK} Biobank",
  author    = "Yang, Ziqi and Song, Ziyang and Zabad, Shadi and Legault,
               Marc-Andr{\'e} and Li, Yue",
  journal   = "Brief. Bioinform.",
  publisher = "Oxford University Press (OUP)",
  volume    =  27,
  number    =  1,
  month     =  jan,
  year      =  2026,
  copyright = "https://creativecommons.org/licenses/by/4.0/",
  language  = "en"
}

\clearpage

%%%%%%%%%%%%%%%%%%%%%%%%%%%%%%%%%%%%%%%%%%%%%%%%%%%%%%%%%%%%

\appendix

\section{Notation Table}

\begin{table}[h]
\centering
\caption{Summary of Notations and Variables used in SMI}
\label{tab:notation}
\begin{tabular}{l|l}
\toprule
\textbf{Notation} & \textbf{Description} \\
\midrule
\multicolumn{2}{c}{\textit{SMI and Embedding}} \\
\midrule
$L$ & Total number of hierarchy levels \\
$l$ & Index of a specific hierarchy level \\
$C^{(\ell)}$ & Set of clinical codes at hierarchy level $\ell$ \\
$c_{j}^{(\ell)}$ & The $j$-th clinical code at hierarchy level $\ell$ \\
$\mathcal{P}(i)$ & Semantic ID of concept $i$ \\
$a_{i}^{(l)}$ & Ancestor code of concept $i$ at level $l$ \\
$\mathbf{e}^{(\ell)}_{j}$ & Embedding of the $j$-th code at level $\ell$\\
$\mathbf{E}^{(\ell)} $ & Embedding matrix for level $\ell$  \\
$\text{Enc}(\cdot)$ & Frozen biomedical LM encoder \\
$\mathbf{e}^{\text{SMI}}_{i}$ & SMI embedding  \\
$w_\ell$ & Global weight for hierarchy level $\ell$ \\
$\mathbf{e}_i^{\text{SMI-W}}$ & Weighted SMI embedding \\
\midrule
% \multicolumn{2}{c}{\textit{Hyperbolic Geometry and Contrastive Learning}} \\
% \midrule
% $c$ & Negative curvature of the hyperbolic space ($c > 0$) \\
% $\mathbb{B}_{c}^{d}$ & $d$-dimensional Poincaré ball with curvature $-c$ \\
% $\mathcal{L}_{HCL}$ & Hyperbolic Contrastive Learning loss \\
% $(c, p, n)$ & Triplet: anchor, positive, negative samples \\
% $\alpha$ & Margin hyperparameter for triplet loss \\
% \midrule
\multicolumn{2}{c}{\textit{Theoretical Analysis}} \\
\midrule
$\mathbf{e}_i^{\mathrm{LM}}$ & Hierarchy-agnostic LM embedding  \\
$\boldsymbol{\mu}_i$ & True semantic vector for concept $i$ \\
$\boldsymbol{\epsilon}_i$ & Embedding estimation noise vector \\
$n_{i}$ & Frequency of concept $i$ in the training corpus \\
$\mathrm{SNR}(i,j)$ & Pairwise Signal-to-Noise Ratio between concepts $i$ and $j$ \\
$\delta_{c}^{(l+1)}$ & Semantic offset of child concept $c$ at level $\ell+1$ with respect to its parent \\
\bottomrule
\end{tabular}
\end{table}

\clearpage

\section{Theoretical Proofs and Empirical Support for Assumptions}

\label{sec:proof}

\subsection{Estimation Variance under Long-Tailed Concept Distribution (Assumption~\ref{ass:embedding_est})}
\label{app:variance-freq}

We model the learned embedding of a clinical concept $i$, denoted by $\mathbf{e}_i^{\mathrm{LM}} \in \mathbb{R}^d$, as a noisy estimate of its underlying true semantic vector $\boldsymbol{\mu}_i$:
$$
\mathbf{e}_i^{\mathrm{LM}} = \boldsymbol{\mu}_i + \boldsymbol{\epsilon}_i.
$$

\begin{proposition}[Estimator variance decreases with frequency $1/n_i$]\label{pro:noise-var}
Let $\mathbf{e}_i^{\mathrm{LM}}$ denote the empirical estimated embedding of concept $i$ computed from $n_i$ independent samples. We thus have:
\[
\mathbb{E}[\mathbf{e}_i^{\mathrm{LM}}] = \boldsymbol{\mu}_i,
\qquad
\mathbb{E}\left[\left\|\mathbf{e}_i^{\mathrm{LM}} - \boldsymbol{\mu}_i\right\|^2\right]= \frac{\sigma^2 d}{n_i}.
\]

As the frequency $n_i$ decreases, the embedding estimation becomes more unstable
and its variance increases.
\end{proposition}

\paragraph{Proof.}
Let the empirical embedding of concept $i$ be
\[
\mathbf{e}_i^{\mathrm{LM}} = \frac{1}{n_i} \sum_{k=1}^{n_i} \mathbf{x}_k,
\]
where each sample $\mathbf{x}_k \in \mathbb{R}^d$ satisfies
$\mathbb{E}[\mathbf{x}_k] = \boldsymbol{\mu}_i$ and
$\operatorname{Cov}(\mathbf{x}_k) = \Sigma_i$, independently across $k$.

We first show that the empirical mean is an unbiased estimator of the true semantic vector
$\boldsymbol{\mu}_i$, i.e., $\mathbb{E}[\mathbf{e}_i^{\mathrm{LM}}] = \boldsymbol{\mu}_i$, which follows from the linearity of expectation:
\[
\begin{aligned}
    \mathbb{E}[\mathbf{e}_i^{\mathrm{LM}}] &= \mathbb{E}\!\left[\frac{1}{n_i}\sum_{i=1}^{n_i} \mathbf{x}_i \right] \\
    &= \frac{1}{n_i} \sum_{k=1}^{n_i} \mathbb{E}[x_k] \\
    &= \boldsymbol{\mu}_i .
\end{aligned}
\]

Since the samples are independent, cross-covariance terms vanish, and the covariance is given by
\[
\begin{aligned}
    \operatorname{Cov}(\mathbf{e}_i^{\mathrm{LM}}) 
    &= \operatorname{Cov}\!\left(\frac{1}{n_i}\sum_{k=1}^{n_i} \mathbf{x}_k\right) \\
    &= \frac{1}{n_i^2} \sum_{k=1}^{n_i} \operatorname{Cov}(x_k) \\
    &= \frac{1}{n_i}\Sigma_i .
\end{aligned}
\]

Thus the estimator is unbiased, and its variance decreases at rate $1/n_i$. For simplicity and without loss of generality, we assume isotropic noise $\Sigma_i=\sigma^2 I_d$, which matches Assumption~\ref{ass:embedding_est}, establishing Proposition~\ref{pro:noise-var}.
\hfill$\square$

This derivation motivates the formal noise assumption introduced in Assumption~\ref{ass:embedding_est}, which models the frequency-dependent uncertainty of LM embeddings.
\[
\mathbf{e}_i^{\mathrm{LM}} = \boldsymbol{\mu}_i + \boldsymbol{\epsilon}_i,
\qquad
\boldsymbol{\epsilon}_i \sim \mathcal{N}\left( \mathbf{0}, \frac{\sigma^2}{n_i} \mathbf{I}_d \right)
\]

\subsection{Quantifying Noise Difference Between Two Concepts (Definition~\ref{def:pairwise-snr})}
\label{app:noise-diff}

In this subsection, we consider two independent clinical concepts and quantify the expected variability between their noisy embeddings. This result serves to define the pairwise noise scale, as introduced in Definition~\ref{def:pairwise-snr}.

\begin{lemma}[Noise Difference Between Two Concepts]\label{lem:noise-diff}

For two concepts $i$ and $j$, we assume their embedding noise vectors
$\boldsymbol{\epsilon}_i$ and $\boldsymbol{\epsilon}_j$ are independent and satisfy
\[
    \boldsymbol{\epsilon}_c \sim \mathcal{N} \left(\mathbf{0}, \frac{\sigma^2}{n_c}\mathbf{I}_d \right),
    \qquad c\in\{i,j\}.
\]

Then the covariance of the noise difference $\boldsymbol{\epsilon}_i - \boldsymbol{\epsilon}_j$ is
\[
\operatorname{Cov}(\boldsymbol{\epsilon}_i - \boldsymbol{\epsilon}_j)
= \sigma^2\!\left(\frac{1}{n_i} + \frac{1}{n_j}\right)\mathbf{I}_d.
\]

Consequently, the expected squared magnitude of the noise difference satisfies

\[
\mathbb{E}\,\bigl\|\boldsymbol{\epsilon}_i - \boldsymbol{\epsilon}_j\bigr\|^2
= \sigma^2 d\!\left(\frac{1}{n_i} + \frac{1}{n_j}\right),
\]

and the Root Mean Square (RMS) magnitude of the noise difference satisfies
\[
\bigl\|\boldsymbol{\epsilon}_i - \boldsymbol{\epsilon}_j\bigr\|_{\mathrm{RMS}}
:= \sqrt{\mathbb{E}\bigl\|\boldsymbol{\epsilon}_i - \boldsymbol{\epsilon}_j\bigr\|^2}
= \sigma \sqrt{\,d\!\left(\frac{1}{n_i} + \frac{1}{n_j}\right)} \, .
\]
\end{lemma}

\begin{proof} Under Assumption~\ref{ass:embedding_est}, the noise vectors are independent and satisfy:
\[
\operatorname{Cov}(\boldsymbol{\epsilon}_i) = \frac{\sigma^2}{n_i} \mathbf{I}_d,
\qquad
\operatorname{Cov}(\boldsymbol{\epsilon}_j) = \frac{\sigma^2}{n_j} \mathbf{I}_d.
\]

By independence, the covariance of their noise difference is:
\[
\operatorname{Cov}(\boldsymbol{\epsilon}_i - \boldsymbol{\epsilon}_j)
= \operatorname{Cov}(\boldsymbol{\epsilon}_i) + \operatorname{Cov}(\boldsymbol{\epsilon}_j)
= \sigma^2 \left( \frac{1}{n_i} + \frac{1}{n_j} \right) \mathbf{I}_d.
\]

Taking the trace gives the expected squared norm:
\[
\mathbb{E}\,\|\boldsymbol{\epsilon}_i - \boldsymbol{\epsilon}_j\|^2
= \operatorname{Tr}\left( \operatorname{Cov}(\boldsymbol{\epsilon}_i - \boldsymbol{\epsilon}_j) \right)
= \sigma^2 d \left( \frac{1}{n_i} + \frac{1}{n_j} \right).
\]

Finally, taking the square root yields the RMS magnitude:
\[
\|\boldsymbol{\epsilon}_i - \boldsymbol{\epsilon}_j\|_{\mathrm{RMS}}
= \sqrt{\mathbb{E}\,\|\boldsymbol{\epsilon}_i - \boldsymbol{\epsilon}_j\|^2}
= \sigma \sqrt{d \left( \frac{1}{n_i} + \frac{1}{n_j} \right)}.
\]
\end{proof}

\begin{lemma}[Noise Concentration in High Dimensions]\label{lem:noise-conc}
Under Assumption \ref{ass:embedding_est}, for two concepts $i$ and $j$, let
$\boldsymbol{\epsilon}_i \sim \mathcal{N}\!\left(0,\frac{\sigma^2}{n_i}\mathbf{I}_d\right)$ and
$\boldsymbol{\epsilon}_j \sim \mathcal{N}\!\left(0,\frac{\sigma^2}{n_j}\mathbf{I}_d\right)$ be independent.
Then the noise difference norm concentrates around its RMS scale:
$$
\frac{\bigl\|\boldsymbol{\epsilon}_i-\boldsymbol{\epsilon}_j\bigr\|}
{\sigma \sqrt{d\!\left(\frac{1}{n_i}+\frac{1}{n_j}\right)}}
\xrightarrow{p} 1 \qquad (d\to\infty),
$$
and equivalently,
$$
\frac{\sigma \sqrt{d\!\left(\frac{1}{n_i}+\frac{1}{n_j}\right)}}
{\bigl\|\boldsymbol{\epsilon}_i-\boldsymbol{\epsilon}_j\bigr\|}
\xrightarrow{p} 1 \qquad (d\to\infty).
$$
\end{lemma}

\begin{proof}
By Lemma \ref{lem:noise-diff},
\[
\boldsymbol{\epsilon}_i-\boldsymbol{\epsilon}_j
\sim \mathcal{N}\!\left(
0,\;
\sigma^2\Bigl(\tfrac{1}{n_i}+\tfrac{1}{n_j}\Bigr)\mathbf{I}_d
\right),
\]
and
\[
\mathbb{E}\,\|\boldsymbol{\epsilon}_i-\boldsymbol{\epsilon}_j\|^2
=
\sigma^2 d\!\left(\tfrac{1}{n_i}+\tfrac{1}{n_j}\right).
\]

Since the noise difference is an isotropic Gaussian vector, its squared norm divided by its expectation follows a normalized $\chi^2_d$ distribution. Therefore,
\[
\frac{\|\boldsymbol{\epsilon}_i-\boldsymbol{\epsilon}_j\|^2}
{\sigma^2 d\!\left(\tfrac{1}{n_i}+\tfrac{1}{n_j}\right)}
\xrightarrow{p} 1
\qquad (d\to\infty).
\]
Taking square roots yields
\[
\frac{\|\boldsymbol{\epsilon}_i-\boldsymbol{\epsilon}_j\|}
{\sigma \sqrt{d\!\left(\tfrac{1}{n_i}+\tfrac{1}{n_j}\right)}}
\xrightarrow{p} 1.
\]
The reciprocal convergence follows by continuous mapping.
\end{proof}

\subsection{Noise Concentration in High Dimensions (Theorem~\ref{thm:separability_collapse})}
\label{app:noise-concentration}

\begin{proof}[Proof of Theorem~\ref{thm:separability_collapse}]
Let $i$ and $j$ be two low-frequency, fine-grained concepts considered in
Theorem~\ref{thm:separability_collapse}. Under Assumption~\ref{ass:embedding_est},
their hierarchy-agnostic LM embeddings admit the decomposition
\[
\mathbf e_i^{\mathrm{LM}} = \boldsymbol\mu_i + \boldsymbol\epsilon_i,
\qquad
\mathbf e_j^{\mathrm{LM}} = \boldsymbol\mu_j + \boldsymbol\epsilon_j,
\]
where $\boldsymbol\epsilon_i \sim \mathcal N(\mathbf 0, \sigma^2 n_i^{-1}\mathbf I_d)$
and $\boldsymbol\epsilon_j \sim \mathcal N(\mathbf 0, \sigma^2 n_j^{-1}\mathbf I_d)$
are independent.

Recall from Definition~\ref{def:pairwise-snr} that the pairwise signal-to-noise ratio
is defined as
\[
\operatorname{SNR}(i,j)
=
\frac{\|\boldsymbol{\mu}_i - \boldsymbol{\mu}_j\|}
{\sigma \sqrt{d\!\left(\frac{1}{n_i}+\frac{1}{n_j}\right)}} .
\]
By the low-SNR assumption in Theorem~\ref{thm:separability_collapse}, this quantity
satisfies $\operatorname{SNR}(i,j) \ll 1$.

By Lemma~\ref{lem:noise-diff}, the noise difference
$\boldsymbol\epsilon_i - \boldsymbol\epsilon_j$ is an isotropic Gaussian vector with
covariance
$\sigma^2\!\left(\frac{1}{n_i}+\frac{1}{n_j}\right)\mathbf I_d$.
Lemma~\ref{lem:noise-conc} then implies the concentration result
\[
\frac{\|\boldsymbol\epsilon_i - \boldsymbol\epsilon_j\|}
{\sigma \sqrt{d\!\left(\frac{1}{n_i}+\frac{1}{n_j}\right)}}
\xrightarrow{p} 1 ,
\]
and equivalently,
\[
\frac{\sigma \sqrt{d\!\left(\frac{1}{n_i}+\frac{1}{n_j}\right)}}
{\|\boldsymbol\epsilon_i - \boldsymbol\epsilon_j\|}
\xrightarrow{p} 1 .
\]

Combining the above, we write
\[
\frac{\|\boldsymbol\mu_i - \boldsymbol\mu_j\|}
{\|\boldsymbol\epsilon_i - \boldsymbol\epsilon_j\|}
=
\underbrace{
\frac{\|\boldsymbol\mu_i - \boldsymbol\mu_j\|}
{\sigma \sqrt{d\!\left(\frac{1}{n_i}+\frac{1}{n_j}\right)}}
}_{\operatorname{SNR}(i,j)}
\cdot
\underbrace{
\frac{\sigma \sqrt{d\!\left(\frac{1}{n_i}+\frac{1}{n_j}\right)}}
{\|\boldsymbol\epsilon_i - \boldsymbol\epsilon_j\|}
}_{\xrightarrow{p} 1}.
\]

The first factor converges to $0$ by the low-SNR assumption, while the second
factor converges in probability to $1$. By Slutsky's theorem,
\[
\frac{\|\boldsymbol\mu_i - \boldsymbol\mu_j\|}
{\|\boldsymbol\epsilon_i - \boldsymbol\epsilon_j\|}
\xrightarrow{p} 0 .
\]

Therefore, for rare fine-grained concepts, the semantic separation is asymptotically
dominated by estimation noise, and hierarchy-agnostic LM embeddings become
geometrically inseparable, completing the proof.
\end{proof}

\subsection{Separability Against Non-Parent Concepts under SMI (Proof of Proposition~\ref{prop:smi_nonparent_separable})}
\label{app:smi-nonparent-proof}

Recall Assumption~\ref{ass:hier-decom}, which states that for a leaf concept $c$ with parent $p$,
the true semantic representation decomposes as
\[
\boldsymbol{\mu}_c
=
\boldsymbol{\mu}_p
+
\boldsymbol{\delta}_c,
\]
where $\boldsymbol{\delta}_c$ is a child-specific semantic offset.
Under SMI, the embedding of $c$ aggregates representations along its hierarchical path.
Since $p$ is an ancestor of $c$ and receives nonzero weight,
the SMI semantic signal for $c$ satisfies
\[
\boldsymbol{\mu}_c^{\mathrm{SMI}}
=
\boldsymbol{\mu}_p
+
\boldsymbol{\delta}_c^{\mathrm{res}},
\]
where $\boldsymbol{\delta}_c^{\mathrm{res}}$ denotes the residual child-specific component.

Let $j \neq p$ be any other concept such that
$\|\boldsymbol{\mu}_p - \boldsymbol{\mu}_j\| > 0$.
Then the semantic difference between $c$ and $j$ under SMI satisfies
\[
\boldsymbol{\mu}_c^{\mathrm{SMI}} - \boldsymbol{\mu}_j
=
(\boldsymbol{\mu}_p - \boldsymbol{\mu}_j)
+
\boldsymbol{\delta}_c^{\mathrm{res}}.
\]
As $n_c \to 0$, the residual term $\boldsymbol{\delta}_c^{\mathrm{res}}$ remains bounded,
while the dominant semantic separation is given by $\boldsymbol{\mu}_p - \boldsymbol{\mu}_j$.
Hence,
\[
\|\boldsymbol{\mu}_c^{\mathrm{SMI}} - \boldsymbol{\mu}_j\|
\;\longrightarrow\;
\|\boldsymbol{\mu}_p - \boldsymbol{\mu}_j\|.
\]

Under Assumption~\ref{ass:embedding_est}, the SMI noise difference
$\boldsymbol{\epsilon}_c^{\mathrm{SMI}} - \boldsymbol{\epsilon}_j$
has bounded variance because the parent-level noise does not scale with $n_c$.
Therefore,
\[
\|\boldsymbol{\epsilon}_c^{\mathrm{SMI}} - \boldsymbol{\epsilon}_j\|
\;\xrightarrow{p}\;
\|\boldsymbol{\epsilon}_p - \boldsymbol{\epsilon}_j\|.
\]

Combining the above limits yields
\[
\frac{
\|\boldsymbol{\mu}_c^{\mathrm{SMI}} - \boldsymbol{\mu}_j\|
}{
\|\boldsymbol{\epsilon}_c^{\mathrm{SMI}} - \boldsymbol{\epsilon}_j\|
}
\;\xrightarrow{p}\;
\frac{
\|\boldsymbol{\mu}_p - \boldsymbol{\mu}_j\|
}{
\|\boldsymbol{\epsilon}_p - \boldsymbol{\epsilon}_j\|
}
\;>\; 0,
\]
which establishes that SMI embeddings remain separable between $c$ and any non-parent concept $j$ in the long-tail regime.

\hfill$\square$

\subsection{SMI Improves SNR via Variance Reduction (Theorem~\ref{thm:smi_snr_gain} and Corollary~\ref{cor:uniform_smi_snr})}
\label{app:variance-reduction-smi}

For a clinical code $i$ at any hierarchy level $\ell \in \{1, \dots, L\}$, we model its empirical embedding as a noisy estimate of a true semantic vector $\mu_i \in \mathbb{R}^d$:
\[
\mathbf{e}_i^{(\ell)} = \boldsymbol{\mu}_i^{(\ell)} + \boldsymbol{\epsilon}_i^{(\ell)},
\qquad
\boldsymbol{\epsilon}_i^{(\ell)} \sim
\mathcal{N}\!\left(\mathbf{0}, \frac{\sigma^2}{n_i^{(\ell)}} \mathbf{I}_d\right),
\]
where $n_i^{(\ell)}$ is the number of samples and $\sigma^2$ is a constant noise parameter. We assume a long-tail distribution where coarser concepts are observed more frequently than fine-grained ones. Consequently, the noise variance decreases at coarser levels:
\[
n_i^{(\ell)} > n_i^{(L)}
\;\Longrightarrow\;
\frac{\sigma^2}{n_i^{(\ell)}} \mathbf{I}_d
  < \frac{\sigma^2}{n_i^{(L)}} \mathbf{I}_d,
\qquad \forall\,\ell < L.
\]

A hierarchy-agnostic LM encoder maps a fine-grained concept $i$ (at the leaf level $L$) directly to an embedding, which we model as a noisy estimate of its true semantic vector:
\[
\mathbf{e}_i^{\mathrm{LM}} = \boldsymbol{\mu}_i + \boldsymbol{\epsilon}_i,
\qquad
\boldsymbol{\epsilon}_i \sim \mathcal{N}\!\left(\mathbf{0}, \frac{\sigma^2}{n_i}\mathbf{I}_d\right).
\]

Defining the semantic signal margin as $\Delta_{\mathrm{sig}}(i,j)=\|\boldsymbol{\mu}_i-\boldsymbol{\mu}_j\|$, the corresponding pairwise SNR under hierarchy-agnostic LM embeddings is
\[
    \mathrm{SNR}_{\mathrm{LM}}(i,j)
    := \frac{\Delta_{\mathrm{sig}}(i,j)}{\sqrt{\mathbb{E}\,\bigl\|\boldsymbol{\epsilon}_i^{\mathrm{LM}} - \boldsymbol{\epsilon}_j^{\mathrm{LM}}\bigr\|^2}}
= \frac{\|\boldsymbol{\mu}_i - \boldsymbol{\mu}_j\|}
{\sigma \sqrt{\,d\!\left(\frac{1}{n_i} + \frac{1}{n_j}\right)}}.
\]

We first prove the result for the more general weighted SMI formulation, from which uniform SMI follows as a special case. For a code $i$, weighted SMI constructs a hierarchy-aware embedding by aggregating the level-specific LM embeddings of its ancestors along its hierarchical path:
\[
\mathbf{e}_i^{\mathrm{SMI\text{-}W}}
=
\sum_{\ell=1}^{L} w_\ell \, \mathbf{e}^{(\ell)}_{a_i^{(\ell)}},
\qquad
w_{\ell}>0,\quad \sum_{\ell=1}^{L} w_{\ell}=L,
\]
where $w_{\ell}$ is the aggregation weight at level $\ell$. Uniform SMI corresponds to $w_\ell=1$ for all $\ell$.
For notational simplicity, we write 
$\mathbf{e}_i^{(\ell)}$, $\boldsymbol{\epsilon}_i^{(\ell)}$, and $n_i^{(\ell)}$
for the embedding, noise, and frequency of the level-$\ell$ ancestor $a_i^{(\ell)}$, respectively.

The corresponding raw noise term of the weighted sum is
\[
\boldsymbol{\epsilon}_i^{\mathrm{SMI\text{-}W}}
=
\sum_{\ell=1}^{L} 
w_\ell\, \boldsymbol{\epsilon}^{(\ell)}_{i}.
\]

The expected squared norm of this raw noise term is
\[
\mathbb{E}\,\bigl\|
\boldsymbol{\epsilon}_i^{\mathrm{SMI\text{-}W}}
\bigr\|^2
=
\sum_{\ell=1}^{L} 
w_\ell^2\,
\mathbb{E}\,\bigl\|
\boldsymbol{\epsilon}^{(\ell)}_{i}
\bigr\|^2
=
d\sigma^2
\sum_{\ell=1}^{L}
\frac{w_\ell^2}{n_i^{(\ell)}} .
\]

For two codes $i$ and $j$, the pairwise SNR of the weighted SMI embeddings is
\[
\mathrm{SNR}_{\mathrm{SMI\text{-}W}}(i,j)
=
\frac{\Delta_{\mathrm{sig}}(i,j)}
{\sigma \sqrt{
d \left(
\sum_{\ell=1}^{L} \frac{w_\ell^2}{n_i^{(\ell)}}
+
\sum_{\ell=1}^{L} \frac{w_\ell^2}{n_j^{(\ell)}}
\right)
}}.
\]
Since the semantic margin $\Delta_{\text{sig}}(i, j)$ measures the difference of true semantic representations and thus is fixed across estimators. Consequently, comparing the SNRs of hierarchy-aware (weighted) SMI embeddings and hierarchy-agnostic LM embeddings reduces entirely to comparing their noise variances.

Since pairwise SNR is scale-invariant, i.e., $\mathrm{SNR}(\mathbf{e}) = \mathrm{SNR}(c\,\mathbf{e})$ for any scalar $c > 0$, we can analyze a normalized version of SMI-W without changing its SNR. 
For weights satisfying $\sum_{\ell=1}^{L} w_\ell=L$, define $\bar{w}_\ell=w_\ell/L$, so that $\sum_{\ell=1}^{L}\bar{w}_\ell=1$. The normalized embedding $\bar{\mathbf{e}}_i^{\mathrm{SMI\text{-}W}}=\mathbf{e}_i^{\mathrm{SMI\text{-}W}}/L$ has the same pairwise SNR as $\mathbf{e}_i^{\mathrm{SMI\text{-}W}}$, so it suffices to compare its noise with that of the LM embedding.

With this normalization, we continue to use $w_\ell$ to denote the normalized weights, and we use the same SMI-W notation for the normalized embedding in the following variance calculation. For the normalized SMI-W embedding, the noise variance is
\[
\mathbb{E}\bigl\|\boldsymbol{\epsilon}_i^{\mathrm{SMI\text{-}W}}\bigr\|^2
= d\sigma^2 \sum_{\ell=1}^{L} \frac{w_\ell^2}{n_i^{(\ell)}}.
\]

Since $w_\ell \ge 0$, expanding the square gives
\[
\left(\sum_{\ell=1}^{L} w_\ell\right)^2
=
\sum_{\ell=1}^{L} w_\ell^2
+
2\sum_{1\le k<\ell\le L}w_k w_\ell
\ge
\sum_{\ell=1}^{L} w_\ell^2 .
\]
Hence, under the normalized weights $\sum_{\ell=1}^{L} w_\ell=1$,
\[
\sum_{\ell=1}^{L} w_\ell^2
\le
\left(\sum_{\ell=1}^{L}w_\ell\right)^2
=
1,
\]
with strict inequality when positive weights are assigned to more than one hierarchy level.
Combining this with the long-tail assumption $n_i^{(\ell)} > n_i^{(L)}$ for all $\ell<L$, we obtain
\[
\sum_{\ell=1}^{L}
\frac{w_\ell^2}{n_i^{(\ell)}}
<
\frac{\sum_{\ell=1}^{L}w_\ell^2}{n_i^{(L)}}
<
\frac{1}{n_i^{(L)}}.
\]
The strict inequalities hold because normalized SMI-W assigns positive weight to multiple hierarchy levels, including at least one ancestor level that is more frequent than the leaf concept.

so that
\[
\mathbb{E}\bigl\|\boldsymbol{\epsilon}_i^{\mathrm{SMI\text{-}W}}\bigr\|^2
=
d\sigma^2
\sum_{\ell=1}^{L}
\frac{w_\ell^2}{n_i^{(\ell)}}
<
\frac{d\sigma^2}{n_i^{(L)}}
=
\mathbb{E}\bigl\|\boldsymbol{\epsilon}_i^{\mathrm{LM}}\bigr\|^2.
\]

Applying this to both codes $i$ and $j$ gives
    \[
\underbrace{\left(
\sum_{\ell=1}^{L} \frac{w_\ell^2}{n_i^{(\ell)}} 
+
\sum_{\ell=1}^{L} \frac{w_\ell^2}{n_j^{(\ell)}}
\right)}_{\text{normalized SMI-W noise scale}}
<
\underbrace{\left(
\frac{1}{n_i^{(L)}} 
+
\frac{1}{n_j^{(L)}}
\right)}_{\text{LM noise scale}}.
\]

This shows that the normalized SMI-W embedding has a smaller pairwise noise scale than the hierarchy-agnostic LM embedding. Since pairwise SNR is scale-invariant, the unnormalized and normalized SMI-W embeddings have the same SNR. Therefore,the unnormalized SMI-W embeddings achieve higher pairwise SNR than the hierarchy-agnostic LM embedding:
\[
\mathrm{SNR}_{\mathrm{SMI\text{-}W}}(i,j) 
> 
\mathrm{SNR}_{\mathrm{LM}}(i,j).
\]

Uniform SMI is a special case of SMI-W with $w_\ell=1$ for all $\ell$. Therefore, uniform SMI embeddings also achieve higher pairwise SNR than the hierarchy-agnostic LM embedding:
\[
\mathrm{SNR}_{\mathrm{SMI}}(i,j)
>
\mathrm{SNR}_{\mathrm{LM}}(i,j).
\]

% \subsection{LM Outputs Lie in a Poincaré Ball (Proposition~\ref{prop:lm_poincare_ball})}

% \begin{proof} 
% The \texttt{tanh} activation bounds each output dimension to $(-1, 1)$, implying that $\|\mathbf{x}\|_\infty < 1$. Hence,
% \[
%     \|\mathbf{x}\|_2^2 = \sum_{i=1}^d x_i^2 < \sum_{i=1}^d 1 = d.
% \]
% A vector $\mathbf{x}$ lies in $\mathbb{B}_c^d$ if $\|\mathbf{x}\|_2^2 < \frac{1}{c}$. Substituting $c = \frac{1}{d}$, we ensure that $\|\mathbf{x}\|_2^2 < d$ and every embedding naturally resides within the Poincaré ball $\mathbb{B}_c^d$.
% \end{proof}

\subsection{Offset Consistency Analysis for Hierarchical Semantic Decomposition}
\label{app:offset_consistency}

Assumption~\ref{ass:hier-decom} models a child clinical concept as the composition of its parent concept semantics and a child-specific semantic offset, such as a complication, manifestation, or severity. This provides an approximate representation-level decomposition of fine-grained clinical concepts.  To empirically assess this assumption, 
We empirically assess this assumption by testing whether ICD-9 parent-child concept pairs sharing the same clinical modifier produce more consistent embedding offsets than parent-child pairs from different modifier groups. we use the clinical modifier expressed in each ICD-9 child description relative to its parent description as the child-specific semantic offset.  For example, in the pair ``diabetes mellitus'' $\rightarrow$ ``diabetes mellitus with renal manifestations,'' the modifier is ``renal manifestations.'' We then test whether parent-child concept pairs sharing the same modifier produce more consistent embedding offsets than pairs from different modifier groups.

We construct the offset consistency analysis using the  ICD-9 parent-child concept pairs. We define each decimal-level ICD-9 code as a child concept and use its integer-level ICD-9 code as the parent concept. Each ICD-9 code description is encoded using BioBERT with mean pooling over token embeddings. For each valid parent-child pair $(p,c)$, we define the offset:
\[
    \Delta^{\mathrm{LM}}_{p,c}
    = \mathbf{e}^{\mathrm{LM}}_c - \mathbf{e}^{\mathrm{LM}}_p,
\]
where $\mathbf{e}^{\mathrm{LM}}_p$ and $\mathbf{e}^{\mathrm{LM}}_c$ denote the BioBERT embeddings of the parent and child descriptions, respectively. This offset represents the semantic change from a coarser parent concept to a more specific child concept in the LM embedding space.

We group offsets by clinical modifiers extracted from the child description relative to the parent description. Modifier groups are extracted from the ICD-9 child description relative to its parent description using rule-based lexical patterns.  We first identify explicit modifier cues, including \emph{with}, \emph{without}, \emph{due to}, \emph{secondary to}, and \emph{associated with}; when these cues are absent, we extract residual child-description terms after removing parent-description tokens. Extracted phrases are normalized by removing generic terms and canonicalizing common variants such as \emph{manifestation}/\emph{manifestations} and \emph{complication}/\emph{complications}. We further refine broad phrases into clinically specific modifier groups, including infection status, obstruction subtype, hepatic coma status, ophthalmic manifestations, retinal disorders, postpartum complications, delivery-related complications, and loss-of-consciousness status. Vague ICD descriptions containing terms such as \emph{unspecified}, \emph{other}, \emph{NEC}, \emph{NOS}, or \emph{not elsewhere classified} are removed, and unresolved coarse buckets such as \texttt{kidney/renal}, \texttt{circulatory/peripheral}, and \texttt{neurological} are omitted from the strict analysis because they merge multiple clinically distinct semantic changes. For each selected modifier group, we normalize the offset vectors and compute all pairwise cosine similarities among offsets within that group. These pairs serve as positive pairs, since they compare parent-child offsets sharing the same clinical modifier. For the negative pairs, we sample offset pairs from different selected modifier groups. This comparison tests whether parent-child offsets with the same clinical modifier are more directionally consistent than randomly paired offsets with different modifiers in the same embedding space.

After applying these selection criteria, the offset consistency analysis retains 418 modifier groups and 1,414 ICD-9 parent--child offsets, yielding 3,606 positive pairs with the same clinical modifier. The negative set consists of 50,000 randomly sampled offset pairs from different modifier groups. As shown in Table~\ref{tab:offset_consistency_strict}, positive pairs with the same clinical modifier have substantially higher cosine similarity than negative pairs from different modifier groups. The mean cosine similarity is $0.655$ versus $0.080$, with a median of $0.730$ versus $0.068$ and a mean difference of $0.575$ ($p=2.0\times10^{-4}$). Fig.~\ref{fig:offset_consistency_strict} shows that the same-modifier distribution is substantially shifted to the right relative to the different-modifier baseline. The distribution also contains both high-similarity and moderate-similarity regions, suggesting that some modifier families induce highly stable offsets while others remain more heterogeneous. This pattern is expected because ICD-9 modifiers vary in specificity and standardization. Overall, these results show that ICD-9 parent-child concept pairs sharing the same clinical modifier produce more aligned embedding offsets than pairs from different modifier groups. This supports the hierarchical semantic decomposition in Assumption~\ref{ass:hier-decom} that fine-grained clinical concepts can be represented through parent-level semantics plus a child-specific semantic offset.

\begin{table}[t]
\centering
\small
\caption{Strict ICD-9 offset consistency analysis. Same-modifier offsets are substantially more aligned than different-modifier clinical offsets, with a mean cosine difference of $0.575$ and a one-sided $p$-value of $2.0\times10^{-4}$, supporting the hierarchical semantic decomposition assumption.}
\label{tab:offset_consistency_strict}
\begin{tabular}{lcc}
\toprule
\textbf{Pair type} & \textbf{Mean cosine} & \textbf{Median cosine} \\
\midrule
Same modifier & 0.655 & 0.730 \\
Different modifier & 0.080 & 0.068 \\
\bottomrule
\end{tabular}
\end{table}
\begin{figure}[t]
    \centering
    \includegraphics[width=0.75\linewidth]{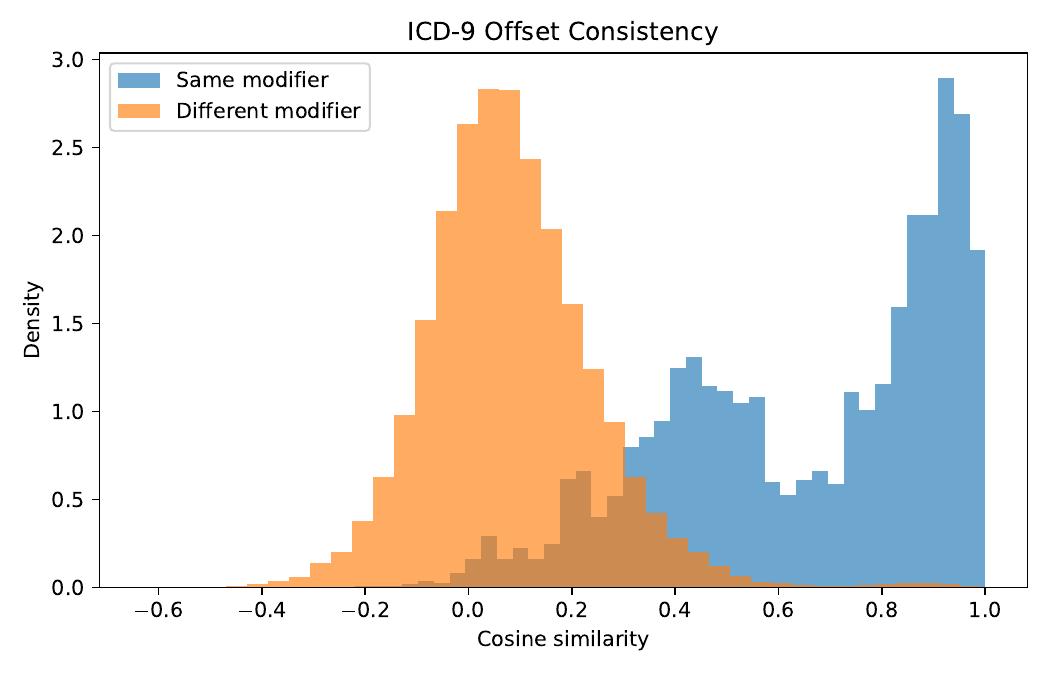}
    \caption{
    Distribution of cosine similarities between ICD-9 parent-child offset vectors in the strict modifier analysis. Same-modifier pairs compare offsets sharing the same extracted clinical modifier, while different-modifier pairs compare offsets from different modifier groups.The same-modifier distribution is shifted strongly to the right, indicating stronger offset consistency than random clinical offsets and supporting the hierarchical semantic decomposition assumption.
    }
    \label{fig:offset_consistency_strict}
\end{figure}

\section{Additional Information for Medical Ontologies}
\label{app: med_onto}
\subsection{Construction of Medical Ontologies}

We explicitly incorporate domain knowledge to represent diagnoses, procedures, and medications using expert-defined, tree-structured medical ontologies. Each clinical concept is associated with a unique hierarchical path, capturing its semantic lineage from coarse clinical categories to fine-grained codes. This one-to-one mapping ensures that each clinical code has a distinct and interpretable position in the taxonomy. 
These hierarchies define the structural relationships used to construct SMIs, allowing embeddings to encode clinically meaningful relationships consistent with expert knowledge.

\paragraph{Diagnosis. } Diagnostic concepts are organized using the CCS developed by the Agency for Healthcare Research and Quality (AHRQ). CCS provides an expert-curated framework that groups ICD-9-CM and ICD-10-CM diagnosis codes into clinically coherent categories at multiple levels of granularity. Based on this structure, we construct a five-level hierarchy comprising major categories, sub-categories, CCS codes, integer-level ICD codes, and decimal-level ICD codes. Each ICD code is uniquely assigned to a single hierarchical path, ensuring an interpretable and unambiguous disease representation.

\paragraph{Procedure.}
Procedure concepts follow the CCS procedure hierarchy, which organizes ICD-9-PCS and ICD-10-PCS codes into a multi-level taxonomy. Specifically, procedures are structured from major categories and sub-categories to CCS procedure codes, followed by three-digit ICD-PCS codes and full ICD-PCS codes. This hierarchy captures progressively finer procedural distinctions while maintaining a unique parent–child relationship at each level, resulting in a well-defined tree structure suitable for embedding learning.

\paragraph{Medication.}
Medication concepts are organized using the ATC classification system, which naturally defines a five-level hierarchy ranging from anatomical main groups to specific chemical substances. This expert-defined structure encodes pharmacological relationships among drugs and provides a clinically grounded hierarchy for representing medication concepts within the SMI framework.

Across diagnoses, procedures, and medications, each clinical code is mapped to a unique hierarchical path, yielding a unified ontology that supports effective and interpretable representation learning with knowledge-based guidance.

\subsection{Statistics of Medical Ontologies}

We adopt expert-defined clinical ontologies to represent diagnoses, procedures, and medications using multi-level hierarchical taxonomies. The summary statistics of clinical hierarchies are reported in Table~\ref{tab:hierarchy_stat}.

\textbf{Diagnosis: } Diagnosis concepts follow the CCS hierarchy, which organizes ICD-9-CM and ICD-10-CM codes into a five-level taxonomy. In total, CCS comprises 18 major disease categories, 134 subcategories, and 265 CCS codes, covering all ICD-9-CM and ICD-10-CM diagnostic concepts. At the fourth level, the hierarchy contains includes 2,175 integer-level ICD codes, comprising 937 ICD-9-CM codes and 1,278 ICD-10-CM codes. At the fifth level, the hierarchy further expands to decimal-level ICD diagnostic codes, including 13,193 ICD-9-CM and 73,066 ICD-10-CM codes, thereby providing the finest granularity.

\textbf{Procedure: } Procedure concepts follow the CCS hierarchy, which organizes ICD-9-PCS and ICD-10-PCS codes into a five-level taxonomy. Specifically, the hierarchy comprises 16 major categories, 200 subcategories, and 224 CCS codes, covering all ICD-9-PCS and ICD-10-PCS procedure codes. At the fourth level, procedure concepts are represented by three-digit ICD-PCS codes, comprising 876 ICD-9-PCS codes and 880 ICD-10-PCS codes. The fifth level corresponds to the full ICD-PCS codes, including 3,674 ICD-9-PCS codes and 80,492 ICD-10-PCS codes.

\textbf{Medication: } Medication concepts follow the ATC classification system, which organizes drug codes into a five-level taxonomy. Specifically, the hierarchy comprises 14 anatomical main groups, 90 therapeutic subgroups, 248 pharmacological subgroups, 840 chemical subgroups, and 5,495 chemical substances (i.e., ATC codes).

\begin{table}[t]
\centering
\caption{Hierarchical taxonomy of diagnosis, procedure, and medication concepts used in this study. Each level contains two attributes: \emph{Name}, which specifies the semantic grouping at that level, and \emph{\# Concepts}, the number of unique codes at that level. Diagnosis categories are derived from ICD-CM codes mapped through the CCS diagnostic hierarchy; procedure categories use ICD-PCS codes organized by the CCS procedure hierarchy; and medications follow the five-level ATC ontology. For ICD-CM and ICD-PCS levels, the \emph{\# Concepts} is reported as the total number of ICD-9 and ICD-10 codes shown in parentheses, respectively.}
\label{tab:hierarchy_stat}
\begin{tabular}{c|c|ccc}
\toprule
\textbf{Level} 
& \textbf{Description} 
& \textbf{Diagnosis} 
& \textbf{Procedure} 
& \textbf{Medication} \\
\midrule
\multirow{2}{*}{\textbf{Level 1}} 
& Name & Major category & Major category  & Anatomical main group \\
& \# Concepts & 18 & 16 & 14 \\
\midrule
\multirow{2}{*}{\textbf{Level 2}} 
& Name & Sub-category & Sub-category & Therapeutic subgroup \\
& \# Concepts & 34 & 200 & 90 \\
\midrule
\multirow{2}{*}{\textbf{Level 3}} 
& Name & CCS code & CCS code & Pharmacological subgroup \\
& \# Concepts & 265 & 224 & 248 \\
\midrule
\multirow{2}{*}{\textbf{Level 4}} 
& Name & ICD code (integer) & ICD code (three-digit)  & Chemical subgroup \\
& \# Concepts &  (937, 1278) &  (876, 880) & 840 \\
\midrule
\multirow{2}{*}{\textbf{Level 5}} 
& Name & ICD code  & ICD code (decimal)  & Chemical substance \\
& \# Concepts & (13193, 73066) & (3674, 80492) & 5495 \\
\bottomrule
\end{tabular}
\end{table}

\section{Additional Experiment Details}

\subsection{MIMIC Dataset Processing}

MIMIC-IV is a publicly available, de-identified EHR database derived from the Beth Israel Deaconess Medical Center (BIDMC) and released through PhysioNet under the PhysioNet Credentialed Health Data License \citep{johnson2023mimic}. The database contains comprehensive longitudinal hospital and intensive care data collected from emergency department and inpatient admissions between 2008 and 2022. In this study, we use the hospital module of MIMIC-IV, which provides structured information on inpatient hospitalizations, including diagnoses, procedures, medication prescriptions, and admission and discharge timestamps. To ensure patient privacy, all identifiers are replaced with de-identified integer keys, and all timestamps are shifted by a patient-specific offset. This transformation preserves within-patient temporal ordering while preventing alignment of absolute calendar time across patients. After preprocessing, our cohort comprises 223,452 unique patients and 546,028 hospital admissions. On average, each patient has 2.44 hospital visits, forming longitudinal trajectories suitable for sequential modeling.

\begin{table}[h]
\centering
\caption{Basic statistics of the pre-processed MIMIC-IV EHR dataset.
``\#'' denotes counts and ``/patient'' denotes per-patient averages.}
\label{tab:mimic4_basic_stats}
\setlength{\tabcolsep}{5pt}
\begin{tabular}{lcccccc}
\toprule
Dataset
& \#patients
& \#visits
& \#visits/patient
& \#conditions/patient
& \#procedures/patient
& \#drugs/patient \\
\midrule
MIMIC-IV
& 223{,}452
& 546{,}028
& 2.44
& 28.48
& 3.85
& 48.32 \\
\bottomrule
\end{tabular}
\end{table}

We represent the EHR data using a visit-based formulation, where each visit corresponds to one inpatient hospitalization and each patient is modeled as a temporally ordered sequence of visits. Visits are ordered chronologically by admission time, while clinical concepts within a visit are treated as unordered sets, reflecting the episodic nature of hospital care. For each visit, we extract diagnostic and procedural information from the MIMIC-IV diagnoses and procedures tables. Diagnostic concepts are represented using the original ICD-9-CM and ICD-10-CM codes, with the coding system preserved via the recorded ICD version. After preprocessing, the dataset contains 9,143 unique ICD-9 diagnosis codes and 19,440 unique ICD-10 diagnosis codes. Procedural concepts are represented using ICD-9 and ICD-10-PCS procedure codes, yielding 2,557 unique ICD-9 procedure codes and 12,354 unique ICD-10 procedure codes in total. Within each visit, diagnoses and procedures are treated as unordered sets, while visit sequences for each patient preserve temporal order across admissions. Across the cohort, patients have on average 28.48 diagnosis codes and 3.85 procedure codes, aggregated over all visits.

Medication information is extracted from the MIMIC-IV prescriptions table, where drugs are originally recorded using National Drug Codes (NDC). Because NDC identifiers are highly granular and introduce substantial sparsity, we standardize medication concepts through a two-stage mapping process. Each NDC is first mapped to a corresponding RxNorm concept identifier (RxCUI), and each RxCUI is subsequently mapped to an Anatomical Therapeutic Chemical (ATC) classification code using a curated RxNorm–ATC reference. Unlike common practice that aggregates medications at higher ATC levels, we retain ATC level-5 codes, which correspond to specific chemical substances and preserve fine-grained pharmacological information. Not all NDC entries can be reliably mapped through this pipeline, for example due to non-drug supplies, institution-specific packaging codes, or missing references in external vocabularies. Following a strict standardization strategy consistent with prior work, unmappable NDCs are excluded from the medication representation. As a result, each visit’s medication set consists only of successfully mapped ATC level-5 codes. After preprocessing, the final medication vocabulary contains 1,286 unique ATC level-5 codes. Based on the processed dataset, approximately 85\% of hospital admissions contain at least one mapped ATC-5 medication, computed as the fraction of visits with a non-empty ATC-5 medication set. Across all patients, the average number of medication codes per patient is 48.32, aggregated over their longitudinal visit histories.

Based on the visit-based EHR representation, we construct training samples for multiple clinical prediction tasks using longitudinal sequences of inpatient hospitalizations. For each patient with a chronologically ordered visit sequence $\left\{x_1, x_2, \ldots, x_t\right\}$, we generate a set of samples $\left\{\left(x_1\right),\left(x_1, x_2\right), \ldots,\left(x_1, x_2, \ldots, x_t\right)\right\}$, where each sample corresponds to a prefix of consecutive visits. Each visit $x_i$ is represented as an unordered set of diagnosis, procedure, and medication codes, while temporal dependencies are captured exclusively through the ordering of visits within the sequence. This unified formulation enables different clinical prediction tasks to share a consistent input representation, differing only in their target definitions and the specific visit used for supervision.

We consider four clinical prediction tasks. Mortality prediction (MT) aims to predict the in-hospital mortality outcome of the subsequent visit. Formally, given an input sequence $\left(x_1, x_2, \ldots, x_{t-1}\right)$, the model predicts a binary label $y\left[x_t\right] \in\{0,1\}$, where $y\left[x_t\right]=1$ indicates that the patient died during visit $x_i$; the final visit of each patient is excluded from training since no future outcome is available. Readmission prediction (RA) predicts whether the patient will be readmitted within a predefined time window of $\sigma$ days. Using the same input sequence $\left(x_1, x_2, \ldots, x_{t-1}\right)$, the target label is defined as $y\left[\tau\left(x_t\right)-\tau\left(x_{t-1}\right)\right] \in\{0,1\}$, where $\tau\left(x_i\right)$ denotes the admission time of visit $x_i$, and $y=1$ if $\tau\left(x_t\right)-\tau\left(x_{t-1}\right) \leq \sigma$, and $0$ otherwise. We set $\sigma = 15$ days. The evaluation metrics is AUPRC and AUROC. Drug recommendation is formulated as a multi-label classification task, where the input sequence $\left(x_1, x_2, \ldots, x_t\right)$ is used to predict the set of medications prescribed during visit $x_t$. The target is represented as a multi-hot vector over the ATC level-5 medication vocabulary, whose dimensionality equals the number of unique drug types.

\subsection{eICU Dataset Processing}

The eICU Collaborative Research Database is a multi-center intensive care unit (ICU) database containing over 200,000 admissions from ICUs monitored by eICU programs in the United States \citep{eicu}. It offers de-identified EHR data, encompassing demographics, diagnoses, treatments, and interventions. We use the eICU dataset to evaluate distribution shifts in patient observations across hospitals using both SMI and BioBERT embedding methods. We map diagnoses and medications to 72 unique integer-level ICD-9 codes and 241 unique five-level ATC codes. We treat each patient's medical records as unordered observations without considering temporal order. We select the 10 hospitals with the largest number of patients, resulting in 2,418 patients per hospital and a total of 24,179 patients for analysis.

In eICU, approximately one-third of drug names are not recorded across the 208 hospitals (Fig.~\ref{fig:drug_imputation}a). This results in an extremely sparse input and thus hampers the model training due to the lack of informative features. To address this issue, we imputed the missing drug names using the Hierarchical Ingredient Code List (HICL) codes, a standardized coding system for healthcare products. Specifically, for entries without a drug name but with a HICL code, we retrieved the corresponding name by matching to other records that share the same HICL code and have a documented name.  Using the HICL-based imputation strategy, the proportion of unrecorded drugs in our dataset was dropped to around 20\% (Fig.~\ref{fig:drug_imputation}b), providing denser medication inputs for downstream modeling.

\begin{figure}[ht]
    \centering
    \includegraphics[width=0.9\linewidth]{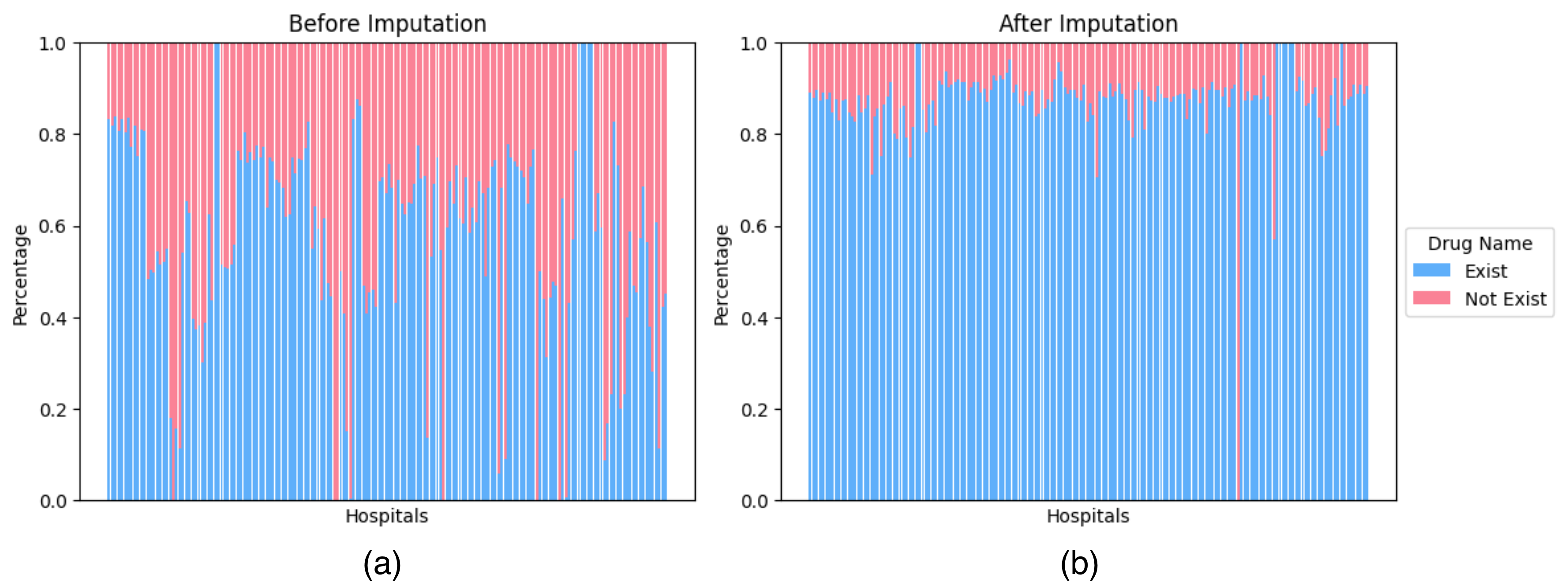}
    \caption{
    The drug rates before and after imputation of the eICU data. Since some drugs may lack recorded names, drug rates indicate the proportion of administered drugs with recorded names over all unique drugs.
    }
    \label{fig:drug_imputation}
\end{figure}

Drug names in eICU are recorded as free-text strings with dosage information, units, and other administration-related details. Consequently, the same active ingredient can appear in many lexical variants, potentially introducing irrelevant noise for training. To address this issue, we applied a data cleaning procedure to the drug-name strings. Specifically, we retained the tokens corresponding to the core medications and removed administration-specific details such as dosages, units, and numeric values. The resulting drug texts more closely reflect the medication's active-ingredient terms, supporting more stable and transferable medication representations.

We also found that the same medication can appear under different names in the eICU dataset (e.g., trade/brand names vs. ingredient names), creating multiple lexical variants of the same drug entity. This variability inflates the effective vocabulary and introduces spurious distinctions that reduce model robustness. To mitigate this issue, we harmonized eICU drug strings by mapping them to standardized medication concepts using the WHO Anatomical Therapeutic Chemical (ATC) classification system. Specifically, we first collected ATC classes from the WHO Collaborating Centre for Drug Statistics Methodology website \citep{kury_atcd_2024}. We then leveraged the five-level ATC hierarchy and extracted the 5th-level ATC substance codes as well as their descriptions. For each eICU canonical drug name, we computed BioWordVec embeddings using for both the eICU string and each reference ATC substance name \citep{zhang_2019}. We then measured the cosine similarity between these two word embeddings. After comparing with all reference drug names, the reference ATC substance name with the highest cosine similarity is selected. We also assigned the corresponding ATC code to each medication record. This procedure effectively matches drugs with similar ATC codes, improves the consistency of medication concepts across records, and supports standardized analyses across ATC levels.

\subsection{Prompt for Biomedical LMs}

Specifically, we use a frozen Biomedical LM model as an encoder. For each medical concept, we first generate the description of the clinical code at every level and explicitly concatenate its full hierarchical path, incorporating the descriptions of all ancestor codes in the hierarchy. We find that the hierarchy-aware input is more effective in capturing medical semantics. Table~\ref{tab:lm_prompt_inputs} presents the BioBERT input prompt, in which the text passed to the LM encoder explicitly encodes the hierarchical context of a clinical concept.

\begin{table}[t]
\centering
\caption{Level-specific textual descriptions used by the biomedical LM. The example includes an ICD-10 diagnosis (\textit{I15.1 Hypertension secondary to other renal disorders}), an ICD-10 procedure (\textit{79658 Magnetic Resonance Imaging (MRI) of Lumbar Spine}), and an ATC medication (\textit{A10AE04 Insulin glargine}). For each hierarchy level, we encode only the textual description of the code at that level, rather than concatenating descriptions along the full ontology path.}
\label{tab:lm_prompt_inputs}
\footnotesize
\begin{tabular}{
>{\raggedright\arraybackslash}p{1.3cm}|
>{\raggedright\arraybackslash}p{3.4cm}
>{\raggedright\arraybackslash}p{4.0cm}
>{\raggedright\arraybackslash}p{3.8cm}
}
\toprule
\textbf{Level} & \textbf{Diagnosis} & \textbf{Procedure} & \textbf{Medication}  \\
\midrule
\textbf{Level 1} & 
Diseases of circulatory system 
& Miscellaneous diagnostic and therapeutic procedures
& Alimentary tract and metabolism \\
\midrule
\textbf{Level 2} & 
Hypertension 
& Magnetic resonance imaging
& Drugs used in diabetes \\
\midrule
\textbf{Level 3} & 
Hypertension with complications and secondary hypertension 
& Magnetic resonance imaging 
& Insulins and analogues \\
\midrule
\textbf{Level 4} & 
Secondary hypertension  
& Imaging, Axial Skel, Exc Skull Facial, MRI Imaging, Axial Skeleton, Except Skull and Facial Bones, Magnetic Resonance Imaging (MRI)
& Insulins and analogues for injection, long-acting \\
\midrule
\textbf{Level 5} & 
Hypertension secondary to other renal disorders 
& Magnetic Resonance Imaging (MRI) of Lumbar Spine
& Insulin glargine\\
\bottomrule
\end{tabular}
\end{table}

\subsection{Baseline Details}
\label{app: baselines}

We compare SMI against representative baseline methods for learning medical concept embeddings. These baselines differ in how they incorporate textual information, ontology structure, and embedding geometry, providing complementary points of comparison for evaluating hierarchy awareness and robustness under long-tailed concept distributions. We include both concept-level and patient-level representation baselines, as several prior clinical models derive concept embeddings implicitly through task-driven training.

\paragraph{BioBERT.}
BioBERT encodes clinical concepts by embedding their textual descriptions using a pretrained biomedical language model \citep{lee2020biobert}, yielding hierarchy-agnostic Euclidean representations. While effective for frequent concepts, BioBERT does not explicitly model hierarchical relationships among medical codes and tends to degrade for sparse, fine-grained concepts.

% \paragraph{BioMistral.}
% BioMistral is a biomedical large language model adapted from the Mistral architecture and trained on large-scale biomedical corpora to improve clinical and scientific language understanding \citep{labrak2024BioMistral}. Similar to other LM-based approaches, BioMistral produces hierarchy-agnostic Euclidean embeddings and does not explicitly encode structured relations among medical concepts.

\paragraph{SapBERT.}
SapBERT learns biomedical concept representations by aligning synonymous terms in UMLS using contrastive learning over pretrained biomedical language models \citep{liu2021learning}. By enforcing proximity among synonym sets, SapBERT improves semantic consistency and entity normalization performance. However, its supervision is limited to local synonym relations and does not explicitly model hierarchical or multi-level ontology structure, making it insufficient for preserving coarse-to-fine semantics across medical taxonomies.

\paragraph{G-BERT.}
G-BERT incorporates medical ontology information into a Transformer model by augmenting patient records with graph-based contextual representations and pretraining objectives \citep{shang2019pre}. While effective for medication recommendation, G-BERT focuses on patient-level sequence modeling and does not produce standalone hierarchy-aware concept embeddings.

\paragraph{HiTs.}
HiTs treats pretrained language models as implicit hierarchy encoders by supervising them to produce representations that respect hierarchical relations among concepts \citep{he2024language}. While effective at aligning LM embeddings with taxonomic structure through structured objectives, HiTs relies on learning hierarchy consistency during training and does not explicitly aggregate multi-level ancestor context to form a single coarse-to-fine concept embedding.

\paragraph{OnT.}
OnT treats pretrained language models as ontology encoders by fine-tuning them under geometric constraints in hyperbolic space to preserve logical relations defined in OWL ontologies \citep{OnT}. By jointly incorporating textual labels and formal ontology axioms, OnT achieves strong hierarchy consistency and supports ontology-level reasoning tasks. In contrast to SMI, OnT focuses on enforcing global logical structure through geometric modeling, rather than aggregating multi-level semantic signals along ontology paths for downstream clinical representation learning.

\paragraph{GRAM.}
GRAM learns medical concept representations by recursively aggregating ancestor information in a medical ontology using an attention mechanism \citep{choi2017gram}. By leveraging hierarchical structure, GRAM improves embeddings for rare concepts, but it is trained in a supervised, task-specific manner rather than producing general-purpose concept embeddings.

% \paragraph{HiTANet.}
% HiTANet models longitudinal EHR data using hierarchical time-aware attention to capture both visit-level and code-level importance patterns \citep{luo2020hitanet}. Although it improves clinical prediction performance, HiTANet is designed for temporal patient modeling rather than reusable medical concept representation learning.

\paragraph{cui2vec.}
cui2vec learns distributed embeddings of medical concepts by applying word embedding methods to large-scale clinical notes, claims data, and biomedical literature \citep{beam2020clinical}. These embeddings capture statistical co-occurrence patterns but do not explicitly encode hierarchical or relational structure from medical ontologies.

\paragraph{RotatE.}
RotatE embeds entities in complex space and models relations as rotations, enabling it to capture symmetry, antisymmetry, and inversion patterns in knowledge graphs \citep{sun2019rotate}. While more expressive than TransE, RotatE does not explicitly encode hierarchical depth or ontology-defined parent--child semantics.

% \paragraph{Transformer.}
% The Transformer baseline models patient records using a standard self-attention architecture without domain-specific pretraining or ontology integration \citep{vaswani2017attention}. This baseline serves as a generic sequence modeling reference and does not incorporate biomedical language knowledge or hierarchical structure.

\paragraph{BEHRT.}
BEHRT is a Transformer-based model pretrained on longitudinal EHR sequences using diagnosis codes, age, and temporal embeddings to learn patient representations \citep{li2020behrt}. While effective for clinical risk prediction, BEHRT is designed for patient-level modeling rather than reusable medical concept embeddings.

\paragraph{Med-BERT.}
Med-BERT pretrains a BERT-style model on structured EHR data using masked code prediction and visit-level objectives to capture medical context \citep{rasmy2021med}. Although Med-BERT improves downstream clinical tasks, it focuses on contextual patient representations rather than explicit hierarchy-aware concept encoding.

\paragraph{ClinicalBERT.}
ClinicalBERT adapts BERT to the clinical domain by pretraining on MIMIC clinical notes to learn contextualized representations of medical language \citep{alsentzer2019publicly}. These embeddings capture rich clinical semantics from text but do not explicitly encode medical ontologies or hierarchical relations among concepts.

\paragraph{BERGAMOT.}
BERGAMOT learns biomedical entity representations by combining LM-based concept encoding with graph-augmented representation learning \citep{sakhovskiy-et-al-2024-bergamot}. It jointly leverages textual concept information and biomedical graph structure to improve entity representations. This requires additional graph-based representation learning and does not explicitly construct a coarse-to-fine concept representation by aggregating level-specific embeddings along an expert-defined ontology path.

\paragraph{BoxLM.}
BoxLM is a recent ontology-aware clinical prediction model that combines LM-based semantic representations with box embeddings to capture hierarchical structure for diagnosis prediction \citep{tan2025boxlm}. It represents medical concepts using box geometry and incorporates ontology-driven structure into a task-specific prediction framework.  BoxLM is optimized for supervised prediction and does not provide a training-free, task-agnostic concept representation constructed from frozen LM embeddings along ontology paths.

\subsection{Computational resources}

All experiments were conducted on an NVIDIA RTX PRO 6000 GPU with 96GB memory. SMI is training-free and only requires a one-time computation of frozen LM embeddings, after which the embeddings are reused in downstream tasks. Therefore, SMI introduces minimal additional hardware or computational overhead beyond the baseline embedding extraction and downstream evaluation.

\section{Additional Results}

\subsection{Long Tail Distribution of Clinical Concepts}
\label{app:long_tail}

\begin{figure}[!htbp]
    \centering
    \includegraphics[width=0.8\linewidth]{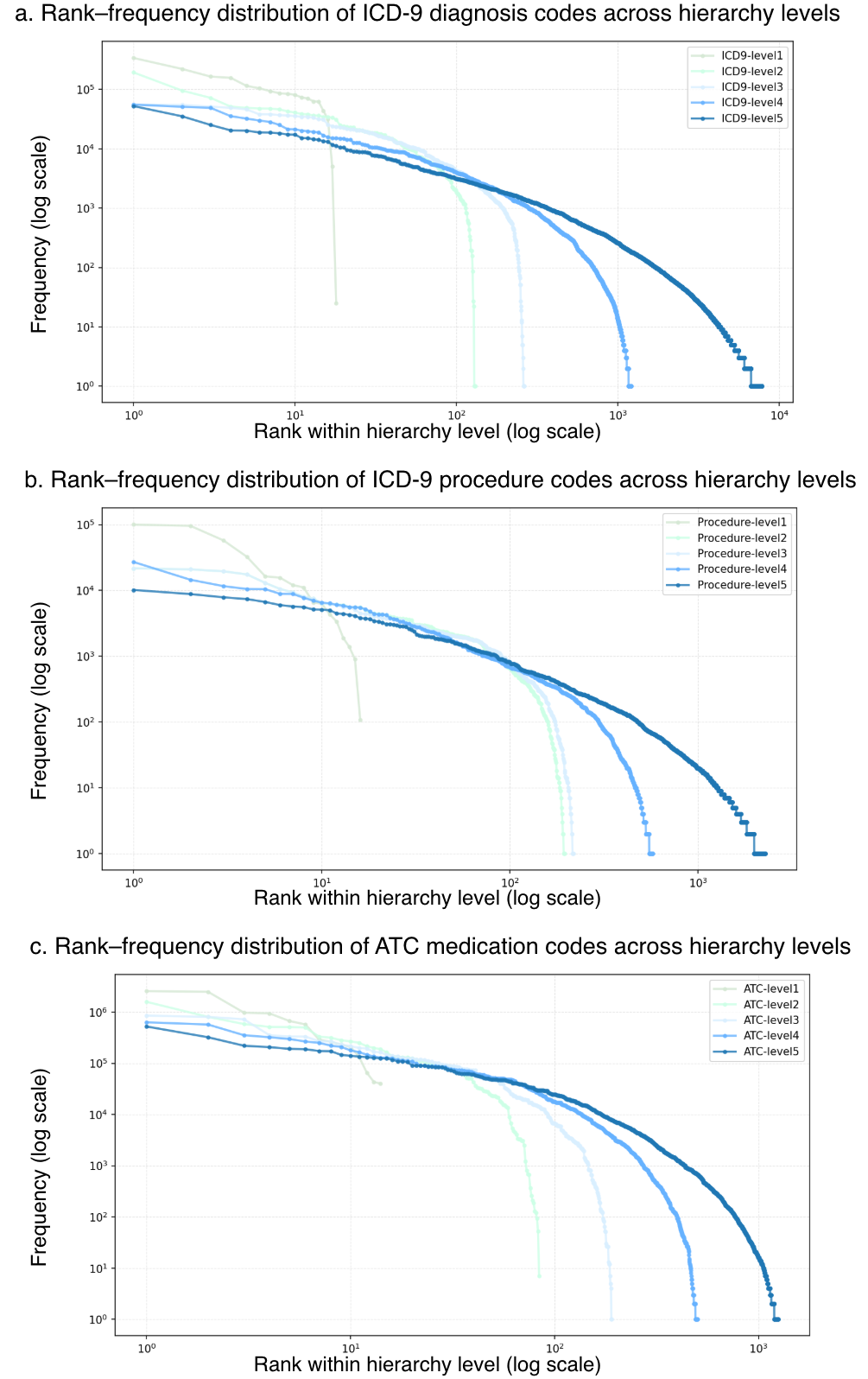}
    \caption{
    Rank--frequency distributions (log--log scale) of clinical concepts in MIMIC-IV across five hierarchy levels:
    (a) ICD-9 diagnosis codes, (b) ICD-9 procedure codes, and (c) ATC medication codes. 
    For each hierarchy level, concepts are ranked by patient frequency within that level.
    }
    \label{fig:long_tail_mimic4}
\end{figure}

Clinical concepts in EHRs typically exhibit highly skewed, long-tailed distributions: a small subset of clinical codes appear frequently, while most codes are observed with low frequencies. To quantify this pattern in MIMIC-IV dataset, we compute patient-level frequencies for three coding systems: ICD-9 diagnoses, ICD-9 procedures, and ATC medications. For each hierarchy, we measure the number of unique patients associated with each concept at all five levels, rank concepts by frequency, and plot the resulting rank–frequency curves on log–log scales.

Fig.~\ref{fig:long_tail_mimic4} summarizes the rank--frequency distributions across hierarchy levels for ICD-9 diagnoses, ICD-9 procedures, ATC medications. A consistent long-tail pattern emerges in all three hierarchies: coarse-grained levels include fewer concepts with higher patient frequencies, whereas fine-grained levels contain many more concepts whose frequencies drop sharply, yielding a large fraction of rare codes. This significant sparsity at fine granularity makes it difficult for hierarchy-agnostic LM embeddings to learn meaningful representations from limited observations. These results motivate hierarchy-aware SMI approach that integrates expert-defined hierarchical semantics to improve representations for infrequent, fine-grained clinical concepts.

\subsection{Details about Clinical Relation Prediction}
\label{app:relation_prediction}

We design our multi-hop prediction task within the ICD and ATC hierarchies following the setup of prior work~\citep{pmlr-v80-ganea18a}. The goal is to evaluate whether the model can infer indirect, multi-hop subsumption relations from direct, one-hop subsumptions. For each positive pair where one concept subsumes another, we sample 10 negative parent candidates that are not connected to the child concept via either direct or indirect subsumption, leading to 10 training triplets per positive example
We consider two negative-sampling strategies: (1) hard negatives, chosen from sibling concepts (i.e., concepts sharing the same parent); (2) random negatives, uniformly from the remaining concepts.  When a concept has fewer than 10 siblings, we fill in with random negatives to maintain 1-to-10 ratio of positives to negatives.

For cross-ontology relation prediction~\citep{Zou2022-te}, we evaluate a one-hop relation prediction using  ICD$\rightarrow$ATC-10 pair. Given an ATC medication code, we predict whether a candidate ICD-10 diagnosis code is linked to it. Since each ATC maps to a single ICD-10 label in our dataset, we restrict evaluation to 1-hop prediction and sample one negative ICD-10 code per positive pair.

For both multi-hop prediction within the ICD and ATC hierarchies and the cross-ontology one-hop ICD$\rightarrow$ATC task, we report Precision, Recall, and F-score as the evaluation metrics. We do not use accuracy metric because class imbalance can inflate it; as the large number of negatives can cause misleadingly high scores.

\subsection{Details about Clustering Analysis}
\label{app:clustering}

We evaluate how well the embeddings group clinical codes into highest-level clinical group (i.e., major categories for ICD diagnostic and procedure codes, and Anatomical main groups for ATC medication codes). To do so, we apply K-means clustering  to both BioBERT and SMI embeddings, setting the number of clusters $K$ to match the number of ground-truth categories.  We evaluate clustering quality using three groups of metrics: (1) Normalized Mutual Information (NMI), Adjusted Rand Index (ARI), and Purity measure cluster-label alignment with ICD major categories; (2) Per-category Recall and Precision measure how coherently each major category is grouped and how pure its dominant cluster is; (3)  Silhouette and Intra/Inter Ratio measure whether embeddings from the same major category are compact and separated from other categories.
For all three metric groups, we report the aggregate clustering results in Table~\ref{tab:cluster_aggregate}. For per-category Recall and Precision across the 18 ICD major categories, we report detailed results in Table~\ref{tab:cluster_per_major}.

\paragraph{\textbf{(1) Cluster-label alignment metrics.}}
We first report global clustering metrics that evaluate how well the unsupervised clusters align with ICD major categories. Specifically, we use NMI, ARI, and cluster Purity. NMI measures the mutual information agreement between cluster assignments and major-category labels. ARI measures pairwise agreement between the clustering result and the ground-truth labels after correcting for chance. Purity measures the fraction of samples in each cluster that belong to its dominant major category. Higher values indicate stronger ontology-consistent clustering.

\paragraph{\textbf{(2) Aggregate and per-category Recall and Precision.}}
We next report Recall and Precision to evaluate category-level clustering quality. For each ICD major category, we compute per-category Recall and Precision, and then report their macro-averages as aggregate metrics in Table~\ref{tab:cluster_aggregate}. Detailed per-category results are reported in Table~\ref{tab:cluster_per_major}. Let $c$ index the resulting clusters, and let $c^\ast(m)$ denote the dominant cluster for category $m$, defined as the cluster containing the largest number of codes from $m$.

Per-category Recall measures how well a diagnostic major $m$ is captured by a single cluster. It is , it defined as the proportion of ICD codes from major $m$ that are assigned to its dominant cluster $c^\ast(m)$:
\[
    \mathrm{recall}(m) = P(\text{cluster}=c^*(m) \mid \text{major}=m)
\] 
Higher Recall indicates that codes from the same major category are concentrated in one coherent cluster, while lower Recall indicates that the category is scattered across multiple clusters.

Per-category Precision measures the purity of the dominant cluster corresponding to a diagnostic major $m$. It is defined as the proportion of ICD codes within the dominant cluster $c^*(m)$ that belong to major $m$:
\[
    \mathrm{precision}(m) \;=\; P\bigl(\text{major}=m \mid \text{cluster}=c^\ast(m)\bigr)
\]
Higher Precision indicates that the dominant cluster contains few codes from other major categories, while lower Precision indicates stronger mixing across categories.

\paragraph{\textbf{(3) Embedding-space separation metrics.}}
We also report Silhouette score and Intra/Inter Ratio as global embedding-space separation metrics to assess whether ICD major categories are geometrically separated. Silhouette score measures whether each code is closer to codes in its own cluster than to codes in other clusters, where higher values indicate better separation. Intra/Inter Ratio is defined as the average within-major-category distance divided by the average between-major-category distance. Lower values indicate more compact embeddings within the same major category and stronger separation across different major categories.

Table~\ref{tab:cluster_aggregate} shows that SMI improves cluster-label alignment over BioBERT, with higher NMI (0.77 vs. 0.39), ARI (0.56 vs. 0.25), and Purity (0.78 vs. 0.51). This indicates that the clusters induced by SMI embeddings are more consistent with ICD major disease categories. SMI also improves per-category clustering quality, achieving higher average Recall (0.80 vs. 0.41) and Precision (0.61 vs. 0.35), showing that codes from the same major category are grouped more coherently and that the dominant clusters are less mixed with codes from other categories. Table~\ref{tab:cluster_per_major} further shows that this improvement holds across most ICD major categories. Finally, SMI improves embedding-space separation, with a higher Silhouette score (0.52 vs. 0.04) and a lower Intra/Inter Ratio (0.29 vs. 0.82), indicating that codes within the same major category are more compact while different major categories are better separated. Overall, these results support that SMI produces diagnosis embeddings that are more aligned with the ICD ontology than BioBERT embeddings.

\begin{table}[t]
\centering
\small
\caption{Aggregate clustering quality of ICD diagnosis embeddings by major disease category. NMI, ARI, Purity, Recall, Precision, and Silhouette are higher better, while Intra/Inter Ratio is lower better. SMI achieves stronger cluster-label alignment and better geometric separation than BioBERT, indicating more ontology-consistent diagnosis embeddings.}
\label{tab:cluster_aggregate}
\begin{tabular}{lccccccc}
\toprule
\textbf{Method}
& \textbf{NMI $\uparrow$}
& \textbf{ARI $\uparrow$}
& \textbf{Purity $\uparrow$}
& \textbf{Recall $\uparrow$}
& \textbf{Precision $\uparrow$}
& \textbf{Sil. $\uparrow$}
& \textbf{Intra/Inter $\downarrow$} \\
\midrule
SMI
& \textbf{0.77}
& \textbf{0.56}
& \textbf{0.78}
& \textbf{0.80}
& \textbf{0.61}
& \textbf{0.52}
& \textbf{0.29} \\
BioBERT
& 0.39
& 0.25
& 0.51
& 0.41
& 0.35
& 0.04
& 0.82 \\
\bottomrule
\end{tabular}
\end{table}

\begin{table}[h]
\centering
\small
\caption{Per-major clustering recall and precision by ICD major disease category. Recall measures whether codes from the same major category are concentrated in a dominant cluster, while precision measures the purity of that dominant cluster. SMI achieves higher recall and precision than BioBERT for most ICD major categories, indicating more coherent and less mixed ontology-aligned clusters.}
\label{tab:cluster_per_major}
\begin{tabular}{llcccc}
\toprule
\textbf{Major} & \textbf{Major Category} & \multicolumn{2}{c}{\textbf{Recall (\%)}} & \multicolumn{2}{c}{\textbf{Precision (\%)}} \\
\cmidrule(lr){3-4}\cmidrule(lr){5-6}
& & \textbf{SMI} & \textbf{BioBERT} & \textbf{SMI} & \textbf{BioBERT} \\
\midrule
1 & Infectious \& parasitic & \textbf{73.3} & 33.1 & \textbf{100.0} & 45.4 \\
2 & Neoplasms & \textbf{62.8} & 53.7 & \textbf{100.0} & 91.5 \\
3 & Endocrine \& metabolic & \textbf{85.3} & 19.8 & \textbf{21.8} & 8.6 \\
4 & Blood diseases & \textbf{91.7} & 41.7 & \textbf{10.6} & 4.8 \\
5 & Mental disorders & \textbf{89.9} & 58.1 & \textbf{100.0} & 58.7 \\
6 & Nervous system \& sense organs & \textbf{89.5} & 28.4 & \textbf{78.2} & 36.5 \\
7 & Circulatory system & \textbf{85.6} & 28.6 & \textbf{41.7} & 17.1 \\
8 & Respiratory system & \textbf{80.9} & 25.8 & \textbf{24.5} & 11.3 \\
9 & Digestive system & \textbf{91.7} & 29.8 & \textbf{53.2} & 23.3 \\
10 & Genitourinary system & \textbf{85.2} & 20.1 & \textbf{33.2} & 12.6 \\
11 & Pregnancy \& childbirth & 51.2 & \textbf{90.2} & \textbf{85.3} & 79.4 \\
12 & Skin & \textbf{87.6} & 17.7 & \textbf{13.4} & 2.9 \\
13 & Musculoskeletal & \textbf{96.3} & 55.4 & \textbf{100.0} & 63.4 \\
14 & Congenital anomalies & \textbf{93.5} & 44.6 & \textbf{100.0} & 24.4 \\
15 & Perinatal conditions & \textbf{63.0} & 53.5 & \textbf{18.1} & 15.9 \\
16 & Symptoms \& ill-defined & \textbf{45.1} & 31.2 & \textbf{100.0} & 89.5 \\
17 & Injury \& poisoning & \textbf{77.1} & 31.3 & \textbf{18.2} & 5.4 \\
18 & External causes (E-codes) & \textbf{82.8} & 78.5 & \textbf{100.0} & 42.6 \\
\bottomrule
\end{tabular}
\end{table}

Across most majors, SMI achieves high recall, suggesting that codes within each major are concentrated in a dominant cluster rather than scattered across multiple clusters. In particular, SMI reaches recall above 80\% for most ICD major categories, including Endocrine \& metabolic, Blood diseases, Mental disorders, Nervous system \& sense organs, Circulatory system, Respiratory system, Digestive system, Genitourinary system, Skin, Musculoskeletal, Congenital anomalies, and External causes. Moreover, SMI also achieves high precision for many majors, including 100\% precision for Infectious \& parasitic, Neoplasms, Mental disorders, Musculoskeletal, Congenital anomalies, Symptoms \& ill-defined, and External causes. This indicates that the dominant clusters for these majors are highly pure and contain few codes from other majors.

In contrast, BioBERT exhibits lower precision across most majors, implying stronger mixing of majors within the dominant clusters even when recall is moderate. We also observe a few challenging categories where precision remains low despite high recall, such as Blood diseases, Skin, and Injury \& poisoning. This suggests that K-means tends to group these categories with clinically or semantically related categories, leading to mixed dominant clusters. Despite this difficulty, SMI still achieves higher recall and precision than BioBERT for these categories. Overall, the results indicate that hierarchy-aware concept representations form clusters that are more consistent with medical ontologies.

\begin{figure*}[t]
    \centering
    \includegraphics[width=\textwidth]{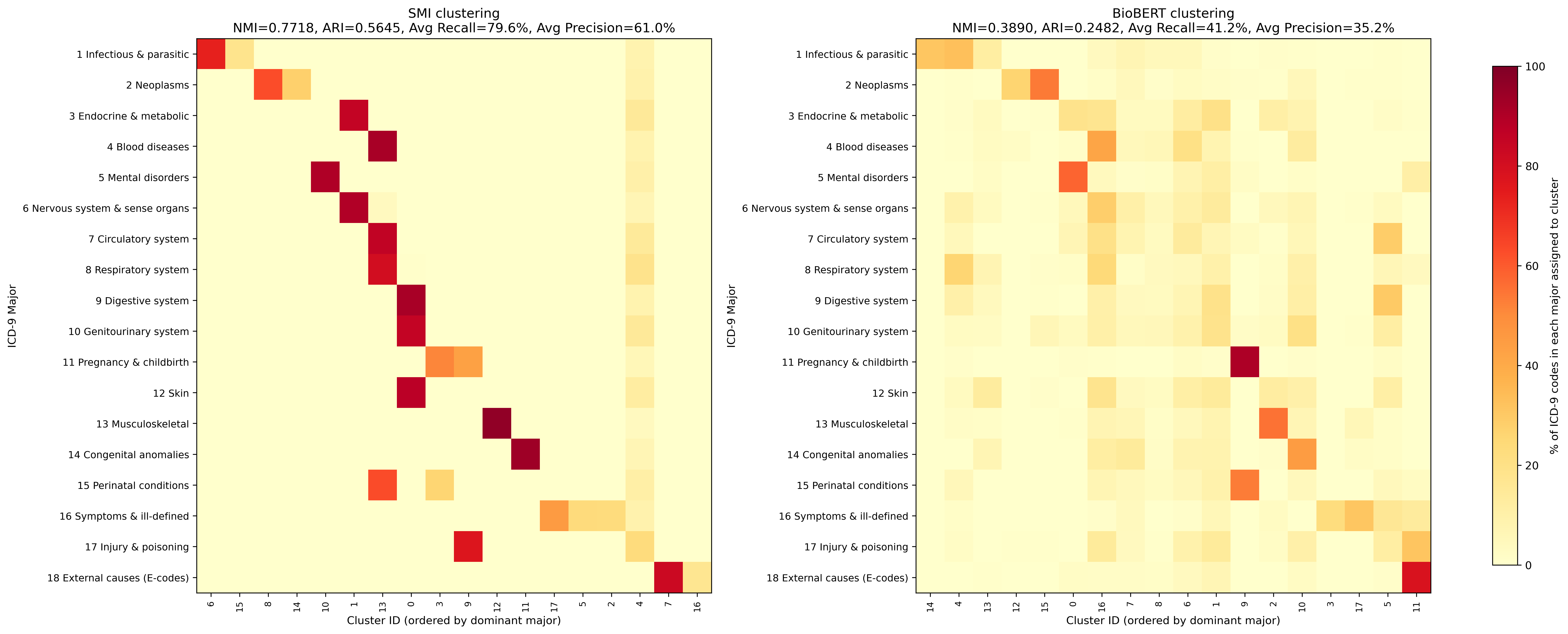}
    \caption{
    Cluster assignment heatmaps for ICD diagnosis embeddings. Rows denote ICD major disease categories and columns denote K-means clusters ordered by their dominant major category. Each cell shows the percentage of codes from a major category assigned to a cluster. Compared with BioBERT, SMI exhibits a stronger diagonal pattern, indicating better alignment between embedding structure and ICD major categories.
    }
    \label{fig:cluster_assignment_heatmap}
\end{figure*}

\subsection{Details about Rare Code Prediction}

To benchmark drug recommendation on rare ATC codes, we construct a new MIMIC-IV subset. We select five-level ATC codes with 10–20 occurrences, yielding 98 unique long-tailed medications, and restrict drug recommendation to this label space. Because rare targets are extremely sparse, the previous threshold-based metrics (AUPRC, F-score, and Jaccard) are less informative; instead, we evaluate retrieval quality using Recall@15, i.e., whether the true medication labels appear within the top-15 prediction set.

\subsection{Details about Distribution shift}
\label{app:dist_shift}

We assess cross-hospital distribution shift following prior work by learning a hospital-specific density model on patient embeddings and measuring the log-likelihood gap between in-domain samples and samples from other hospitals \citep{zhu2025fedweight}. We trained a Variational Autoencoder (VAE) as a density estimator on patient embeddings from the eICU dataset \citep{vae}. For each source hospital $k$, we encoded  diagnoses and medications using either SMI or BioBERT, and then trained a VAE density estimator  $p_k(x_k)$ on the embeddings $x_k$ from hospital $k$. We quantified the distribution shift from source hospital $k$ to target hospital $t$ by computing their difference of log-likelihoods:
\begin{align}
    \Delta = \mathbb{E}_{x_k} \left[ \log p_k(x_k) \right] - \mathbb{E}_{x_t} \left[ \log p_k(x_t) \right]
\end{align}
where $\log p_k(x_k)$ and $\log p_k(x_t)$ are the log‑likelihoods of in‑domain samples $x_k$ and out‑of‑domain samples $x_t$ evaluated under the VAE  trained on $k$-th hospital, respectively. A large $\Delta$ implies greater cross-hospital shift between source and target distributions. We approximate the log-likelihood $\log p_k(x)$ using an Importance Weighted Autoencoder (IWAE) objective by drawing $K=64$ importance samples $z_{i} \sim q(z \mid x)$ \citep{burda2015importance}:
\begin{align}
    \log p_k(x) \approx \log \left( \frac{1}{K} \sum_{i=1}^{K} \frac{p(x, z^{(i)})}{q(z^{(i)} \mid x)} \right)
\end{align}
Fig.~\ref{fig:dist_shifts} visualizes cross-hospital distribution shift as the absolute log-likelihood difference ($|\Delta|$) between each source-target hospital pair using a hospital specific density model trained on the source. SMI yields consistently lower $|\Delta|$ across source–target hospital pairs than BioBERT, indicating better cross-site stability and reduced domain shift. This improvement likely comes from injecting expert-curated medical hierarchies, which place semantically related concepts closer in the embedding space. In contrast, BioBERT relies mainly on textual similarity and does not explicitly encode clinical ontology relationships, leading to larger cross-hospital variability. Overall, SMI effectively mitigates cross-hospital distribution shift and is more suitable for multi-institutional settings with heterogeneous data.

\begin{figure*}[!t]
    \centering
    \includegraphics[width=\textwidth]{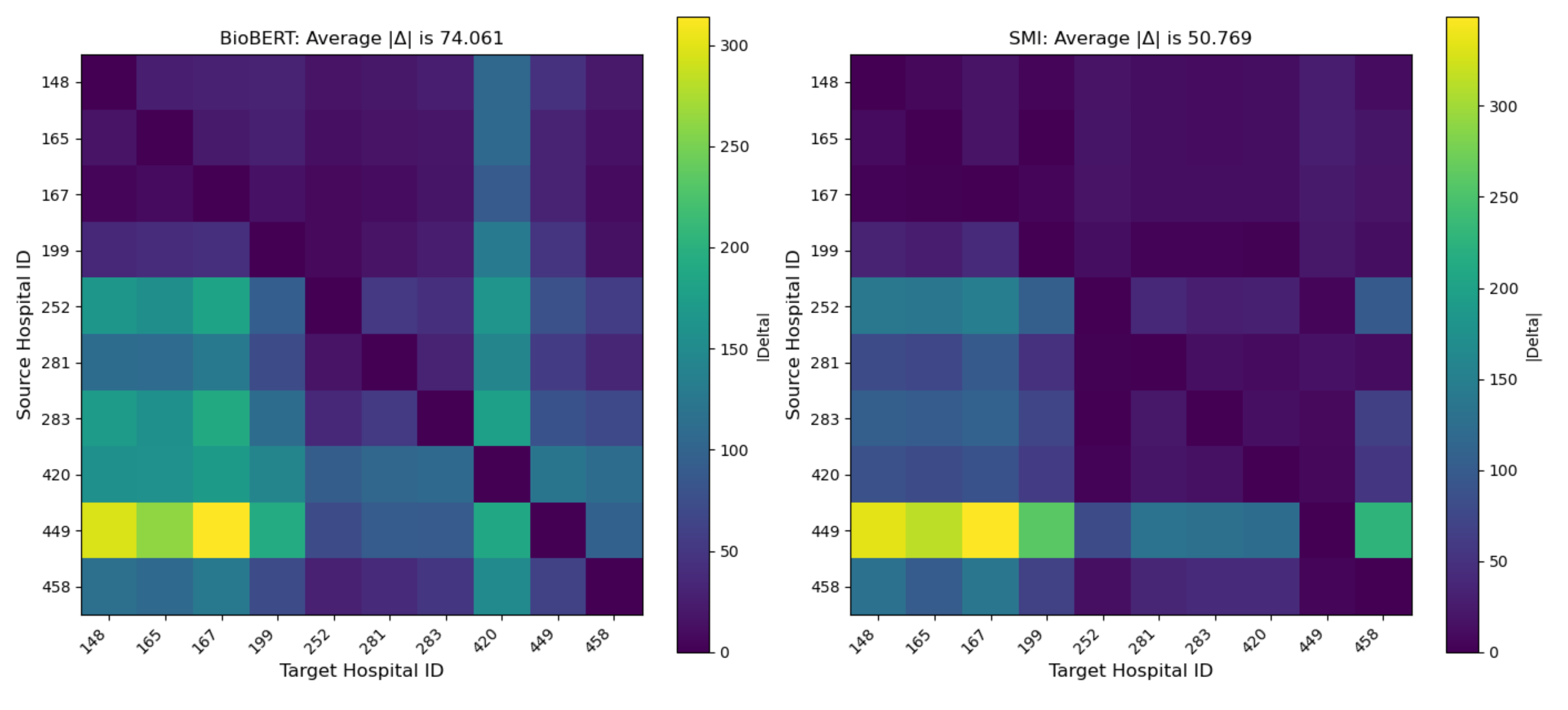} 
    \caption{Cross-hospital distribution shift analysis based on average absolute log-likelihood differences ($|\Delta|$) between source and target hospitals. Each heatmap illustrates the distribution shift of patient embedding using SMI (left) and BioBERT  (right), both evaluated with a VAE-based density model. SMI embeddings exhibit smaller distribution shifts (average $|\Delta| = 50.77$) compared to BioBERT ($|\Delta| = 74.06$), indicating that SMI is more robust to distribution shift. 
    }
    \label{fig:dist_shifts}
\end{figure*}

\clearpage

\newpage
\section*{NeurIPS Paper Checklist}

%%% BEGIN INSTRUCTIONS %%%
The checklist is designed to encourage best practices for responsible machine learning research, addressing issues of reproducibility, transparency, research ethics, and societal impact. Do not remove the checklist: {\bf The papers not including the checklist will be desk rejected.} The checklist should follow the references and follow the (optional) supplemental material.  The checklist does NOT count towards the page
limit. 

Please read the checklist guidelines carefully for information on how to answer these questions. For each question in the checklist:
\begin{itemize}
    \item You should answer \answerYes{}, \answerNo{}, or \answerNA{}.
    \item \answerNA{} means either that the question is Not Applicable for that particular paper or the relevant information is Not Available.
    \item Please provide a short (1--2 sentence) justification right after your answer (even for \answerNA). 
   % \item {\bf The papers not including the checklist will be desk rejected.}
\end{itemize}

{\bf The checklist answers are an integral part of your paper submission.} They are visible to the reviewers, area chairs, senior area chairs, and ethics reviewers. You will also be asked to include it (after eventual revisions) with the final version of your paper, and its final version will be published with the paper.

The reviewers of your paper will be asked to use the checklist as one of the factors in their evaluation. While \answerYes{} is generally preferable to \answerNo{}, it is perfectly acceptable to answer \answerNo{} provided a proper justification is given (e.g., error bars are not reported because it would be too computationally expensive'' or ``we were unable to find the license for the dataset we used''). In general, answering \answerNo{} or \answerNA{} is not grounds for rejection. While the questions are phrased in a binary way, we acknowledge that the true answer is often more nuanced, so please just use your best judgment and write a justification to elaborate. All supporting evidence can appear either in the main paper or the supplemental material, provided in appendix. If you answer \answerYes{} to a question, in the justification please point to the section(s) where related material for the question can be found.

IMPORTANT, please:
\begin{itemize}
    \item {\bf Delete this instruction block, but keep the section heading ``NeurIPS Paper Checklist"},
    \item  {\bf Keep the checklist subsection headings, questions/answers and guidelines below.}
    \item {\bf Do not modify the questions and only use the provided macros for your answers}.
\end{itemize}

%%% END INSTRUCTIONS %%%

\begin{enumerate}

\item {\bf Claims}
    \item[] Question: Do the main claims made in the abstract and introduction accurately reflect the paper's contributions and scope?
    \item[] Answer: \answerYes{} % Replace by \answerYes{}, \answerNo{}, or \answerNA{}.
    \item[] Justification: The abstract and introduction state the main scope and contributions of SMI, including training-free hierarchy-aware concept representation, theoretical SNR analysis, experimental contributions including relation prediction, downstream clinical prediction, and cross-hospital robustness. These claims are supported by the methodology, theory, and experiments reported in the paper.    

\item {\bf Limitations}
    \item[] Question: Does the paper discuss the limitations of the work performed by the authors?
    \item[] Answer: \answerYes{} % Replace by \answerYes{}, \answerNo{}, or \answerNA{}.
    \item[] Justification: The conclusion discusses the scope and limitations of SMI, including its focus on expert-curated coding systems such as ICD and ATC, the need for additional ontology validation and code mapping for emerging disease classifications, newly introduced drugs, and local coding standards.

\item {\bf Theory assumptions and proofs}
    \item[] Question: For each theoretical result, does the paper provide the full set of assumptions and a complete (and correct) proof?
    \item[] Answer: \answerYes{}
    \item[] Justification: The paper states the assumptions and formal theoretical results in Section~\ref{sec: theory}. Complete proofs and supporting derivations are provided in Appendix~\ref{sec:proof}.

    \item {\bf Experimental result reproducibility}
    \item[] Question: Does the paper fully disclose all the information needed to reproduce the main experimental results of the paper to the extent that it affects the main claims and/or conclusions of the paper (regardless of whether the code and data are provided or not)?
    \item[] Answer: \answerYes{}
    \item[] Justification: The paper describes the datasets, ontology construction, data splits, tasks, metrics, baselines, and evaluation protocols in the main text and appendix. These details provide sufficient information to reproduce the main empirical results.

\item {\bf Open access to data and code}
    \item[] Question: Does the paper provide open access to the data and code, with sufficient instructions to faithfully reproduce the main experimental results, as described in supplemental material?
    \item[] Answer: \answerNo{}
    \item[] Justification: The paper uses public but credentialed-access clinical datasets, including MIMIC-IV and eICU, whose redistribution is restricted by data-use agreements. We provide implementation details in the appendix, but the current submission does not include an anonymized code release.

\item {\bf Experimental setting/details}
    \item[] Question: Does the paper specify all the training and test details (e.g., data splits, hyperparameters, how they were chosen, type of optimizer) necessary to understand the results?
    \item[] Answer: \answerYes{}
    \item[] Justification: The paper specifies the clinical prediction tasks, evaluation metrics, data splits, baselines, and implementation details in Section~\ref{sec:relation_pred}, ~\ref{sec: pred_results} and~\ref{sec:dist_shift_main}, with additional baseline and experimental details in Appendices~\ref{app: baselines}, \ref{app:relation_prediction}, and~\ref{app:dist_shift}.    

\item {\bf Experiment statistical significance}
    \item[] Question: Does the paper report error bars suitably and correctly defined or other appropriate information about the statistical significance of the experiments?
    \item[] Answer: \answerYes{}
    \item[] Justification: The clinical prediction results report mean and standard error from bootstrap, while the deterministic relation prediction experiments use threshold selection via grid search and are reported without variance. The paper reports statistical significance tests for drug recommendation, rare drug prediction, and cross-hospital distribution shift. 

\item {\bf Experiments compute resources}
    \item[] Question: For each experiment, does the paper provide sufficient information on the computer resources (type of compute workers, memory, time of execution) needed to reproduce the experiments?
    \item[] Answer: \answerYes{}
    \item[] Justification: All experiments were conducted on an NVIDIA RTX PRO 6000 GPU with 96GB memory. SMI is training-free and only requires a one-time computation of frozen LM embeddings, after which the embeddings are reused in downstream tasks. Therefore, SMI introduces minimal additional hardware or computational overhead.   
    
\item {\bf Code of ethics}
    \item[] Question: Does the research conducted in the paper conform, in every respect, with the NeurIPS Code of Ethics \url{https://neurips.cc/public/EthicsGuidelines}?
    \item[] Answer: \answerYes{}
    \item[] Justification: The study uses de-identified, credentialed-access clinical datasets and does not involve intervention on patients, crowdsourcing, or collection of new human-subject data. We found no deviation from the NeurIPS Code of Ethics.   

\item {\bf Broader impacts}
    \item[] Question: Does the paper discuss both potential positive societal impacts and negative societal impacts of the work performed?
     \item[] Answer: \answerYes{}    
     \item[] Justification: The paper discusses positive impacts through improved representation of long-tailed clinical concepts, better clinical prediction, and robustness to cross-site distribution shift. It also discusses deployment-related risks, including dependence on correct ontology selection, local code mapping, and clinician validation before real-world use. The study uses de-identified clinical datasets and does not release patient-level data.
    
\item {\bf Safeguards}
    \item[] Question: Does the paper describe safeguards that have been put in place for responsible release of data or models that have a high risk for misuse (e.g., pre-trained language models, image generators, or scraped datasets)?
    \item[] Answer: \answerNA{}
    \item[] Justification: The paper does not release a high-risk AI model or  dataset. The clinical datasets used in the study are controlled-access and governed by their original data-use agreements.   
    item[] Guidelines:
        \begin{itemize}
        \item The answer \answerNA{} means that the paper poses no such risks.
        \item Released models that have a high risk for misuse or dual-use should be released with necessary safeguards to allow for controlled use of the model, for example by requiring that users adhere to usage guidelines or restrictions to access the model or implementing safety filters. 
        \item Datasets that have been scraped from the Internet could pose safety risks. The authors should describe how they avoided releasing unsafe images.
        \item We recognize that providing effective safeguards is challenging, and many papers do not require this, but we encourage authors to take this into account and make a best faith effort.
    \end{itemize}

\item {\bf Licenses for existing assets}
    \item[] Question: Are the creators or original owners of assets (e.g., code, data, models), used in the paper, properly credited and are the license and terms of use explicitly mentioned and properly respected?
    \item[] Answer: \answerYes{}
    \item[] Justification: The paper cites the original sources for existing datasets (i.e., MIMIC-IV, eICU), medical ontologies (i.e., ICD, CCS), language models (i.e., BioBERT), and baseline methods. Access to MIMIC-IV and eICU follows their official credentialed-use requirements.    

\item {\bf New assets}
    \item[] Question: Are new assets introduced in the paper well documented and is the documentation provided alongside the assets?
    \item[] Answer: \answerYes{}
    \item[] Justification: We plan to release the precomputed SMI concept embeddings and the code for constructing them from public medical ontologies and frozen biomedical LM encoders. We will provide documentation describing the embedding construction procedure, required ontology resources, and usage instructions for downstream tasks.

\item {\bf Crowdsourcing and research with human subjects}
    \item[] Question: For crowdsourcing experiments and research with human subjects, does the paper include the full text of instructions given to participants and screenshots, if applicable, as well as details about compensation (if any)? 
    \item[] Answer: \answerNA{}
    \item[] Justification: The paper does not involve crowdsourcing, user studies, participant recruitment, or direct interaction with human subjects.    

\item {\bf Institutional review board (IRB) approvals or equivalent for research with human subjects}
    \item[] Question: Does the paper describe potential risks incurred by study participants, whether such risks were disclosed to the subjects, and whether Institutional Review Board (IRB) approvals (or an equivalent approval/review based on the requirements of your country or institution) were obtained?
    \item[] Answer: \answerNA{}
    \item[] Justification: The study uses de-identified, controlled-access EHR datasets and does not collect new human-subject data or interact with participants. Dataset access is governed by the original data-use agreements.

\item {\bf Declaration of LLM usage}
    \item[] Question: Does the paper describe the usage of LLMs if it is an important, original, or non-standard component of the core methods in this research? Note that if the LLM is used only for writing, editing, or formatting purposes and does \emph{not} impact the core methodology, scientific rigor, or originality of the research, declaration is not required.
    \item[] Answer: \answerYes{}
    \item[] Justification: The core method uses a frozen biomedical language model encoder, BioBERT, to obtain level-specific clinical code embeddings. The paper describes this usage in the methodology. 

\end{enumerate}

\end{document}